%% file: main_arxiv.tex
\documentclass[10pt,twocolumn,letterpaper]{article}

\input{preamble_arxiv}
\input{defs.tex}

\begin{document}

\maketitle

\begin{abstract}
Principal component analysis (PCA), along with its extensions to manifolds and outlier contaminated data, have been indispensable in computer vision and machine learning. 
In this work, we present a unifying formalism for PCA and its variants, and introduce a framework based on the flags of linear subspaces, \ie a hierarchy of nested linear subspaces of increasing dimension, which not only allows for a common implementation but also yields novel variants, not explored previously. 
We begin by generalizing traditional PCA methods that either maximize variance or minimize reconstruction error.
We expand these interpretations to develop a wide array of new dimensionality reduction algorithms by accounting for outliers and the data manifold.
To devise a common computational approach, we recast robust and dual forms of PCA as optimization problems on flag manifolds. We then integrate tangent space approximations of principal geodesic analysis (tangent-PCA) into this flag-based framework, creating novel robust and dual geodesic PCA variations. 
The remarkable flexibility offered by the `flagification' introduced here enables even more algorithmic variants identified by specific flag types.
Last but not least, we propose an effective convergent solver for these flag-formulations employing the Stiefel manifold.
Our empirical results on both real-world and synthetic scenarios, demonstrate the superiority of our novel algorithms, especially in terms of robustness to outliers on manifolds.
\end{abstract}

\vspace{-4mm}
\input{content/intro.tex}
\vspace{-1mm}
\input{content/relatedwork_short.tex}

\input{content/background.tex}
\input{content/pca.tex}

\input{content/method.tex}
\input{content/results.tex}

\input{content/conclusion.tex}

\paragraph{Acknowledgements}
N. Mankovich and G. Camps-Valls acknowledge support from the project "Artificial Intelligence for complex systems: Brain, Earth, Climate, Society," funded by the Department of Innovation, Universities, Science, and Digital Society, code: CIPROM/2021/56. T. Birdal acknowledges support from the Engineering and Physical Sciences Research Council [grant EP/X011364/1].

{\small
\bibliographystyle{ieeenat_fullname} 
\bibliography{refs} 
}

\vspace{2mm}
\appendix
\section*{Appendices}

\input{content/appendix.tex}

\end{document}

%% file: preamble_arxiv.tex
\usepackage{cvpr}
\usepackage{times}
\usepackage{epsfig}
\usepackage[utf8]{inputenc}
\usepackage{color}
\usepackage{bm}
\usepackage{amsfonts,mathtools,amsmath, amsthm}
\usepackage{graphicx}
\usepackage{mwe}
\usepackage{makecell}
\usepackage{booktabs}
\usepackage{enumitem}
\usepackage{float} 
\usepackage{graphicx} 
\graphicspath{{images/}} 

\usepackage[ruled,vlined]{algorithm2e} 

\definecolor{cvprblue}{rgb}{0.21,0.49,0.74}
\usepackage[pagebackref,breaklinks,colorlinks,citecolor=cvprblue]{hyperref}

\usepackage{flushend}
\usepackage{setspace}
\usepackage{comment}
\usepackage[capitalize]{cleveref}
\usepackage{multirow}

\def\paperID{824} 
\def\confName{CVPR}
\def\confYear{2024}

\title{Fun with Flags: Robust Principal Directions via Flag Manifolds}

\author{Nathan Mankovich\\
University of Valencia
\and
Gustau Camps-Valls\\
University of Valencia
\and
Tolga Birdal\\
Imperial College London
}

%% file: defs.tex
\newcommand{\R}{\ensuremath{\mathbb{R}}}  
\newcommand{\Gr}{\ensuremath{\mathrm{Gr}}}  

\newcommand{\tr}{\mathrm{tr}}

\newcommand{\bmu}{\bm{\mu}}

\newcommand{\y}{\mathbf{y}}
\newcommand{\Y}{\mathbf{Y}}
\newcommand{\X}{\mathbf{X}}
\newcommand{\Pm}{\mathbf{P}}
\newcommand{\W}{\mathbf{W}}
\newcommand{\x}{\mathbf{x}}

\newcommand{\U}{\mathbf{U}}
\newcommand{\M}{\mathbf{M}}
\newcommand{\bu}{\mathbf{u}}
\newcommand{\vb}{\mathbf{b}}
\newcommand{\bv}{\mathbf{v}}

\newcommand{\V}{\mathbf{V}}

\newcommand{\B}{\mathbf{B}}

\newcommand{\z}{\mathbf{z}}

\newcommand{\I}{\mathbf{I}}

\newcommand{\Z}{\mathbf{Z}}

\newcommand{\flag}{\mathcal{FL}}

\newcommand{\Man}{\mathcal{M}}
\newcommand{\SubMan}{\mathcal{H}}

\newcommand{\TxM}{\mathcal{T}_{\x}\Man}
\newcommand{\TmuM}{\mathcal{T}_{\bm{\mu}}\Man}

\newcommand{\Exp}{\mathrm{Exp}}
\newcommand{\Log}{\mathrm{Log}}

\newcommand{\PCA}{PCA}
\newcommand{\RPCA}{RPCA}
\newcommand{\WPCA}{WPCA}
\newcommand{\DPCP}{DPCP}
\newcommand{\WDPCP}{WDPCP}
\newcommand{\LtwoRPCA}{$L_2$-RPCA}
\newcommand{\LoneRPCA}{$L_1$-RPCA}
\newcommand{\LpRPCA}{$L_p$-RPCA}
\newcommand{\LtwoWPCA}{$L_2$-WPCA}
\newcommand{\LoneWPCA}{$L_1$-WPCA}
\newcommand{\OPCA}{$\perp$PCA}
\newcommand{\LtwoDPCP}{$L_2$-DPCP}
\newcommand{\LtwoWDPCP}{$L_2$-WDPCP}
\newcommand{\LoneDPCP}{$L_1$-DPCP}
\newcommand{\LoneWDPCP}{$L_1$-WDPCP}
\newcommand{\fPCA}{fPCA}
\newcommand{\fOPCA}{f$\perp$PCA}
\newcommand{\fRPCA}{fRPCA}
\newcommand{\fWPCA}{fWPCA}
\newcommand{\fDPCP}{fDPCP}
\newcommand{\fWDPCP}{fWDPCP}
\newcommand{\ftwoRPCA}{fRPCA$(k)$}
\newcommand{\ftwoWPCA}{fWPCA$(k)$}
\newcommand{\ftwoDPCP}{fDPCP$(k)$}
\newcommand{\foneRPCA}{fRPCA$(1,...,k)$}
\newcommand{\foneWPCA}{fWPCA$(1,...,k)$}
\newcommand{\foneDPCP}{fDPCP$(1,...,k)$}

\newcommand{\TPCA}{$\mathcal{T}$PCA}

\newcommand{\TRPCA}{R$\mathcal{T}$PCA}
\newcommand{\TWPCA}{W$\mathcal{T}$PCA}
\newcommand{\fTRPCA}{fR$\mathcal{T}$PCA}
\newcommand{\fTWPCA}{fW$\mathcal{T}$PCA}
\newcommand{\fTDPCP}{f$\mathcal{T}$DPCP}
\newcommand{\fTPCA}{f$\mathcal{T}$PCA}

\newcommand{\LtwoTRPCA}{$L_2$--R$\mathcal{T}$PCA}
\newcommand{\LoneTRPCA}{$L_1$--R$\mathcal{T}$PCA}
\newcommand{\LtwoTWPCA}{$L_2$--W$\mathcal{T}$PCA}
\newcommand{\LoneTWPCA}{$L_1$--W$\mathcal{T}$PCA}
\newcommand{\OTPCA}{$\perp\mathcal{T}$PCA}
\newcommand{\TDPCP}{$\mathcal{T}$DPCP}
\newcommand{\TWDPCP}{W$\mathcal{T}$DPCP}
\newcommand{\LtwoTDPCP}{$L_2$--$\mathcal{T}$DPCP}
\newcommand{\LtwoTWDPCP}{$L_2$--W$\mathcal{T}$DPCP}
\newcommand{\LoneTDPCP}{$L_1$--$\mathcal{T}$DPCP}
\newcommand{\LoneTWDPCP}{$L_1$--W$\mathcal{T}$DPCP}
\newcommand{\DPGP}{DPGP}
\newcommand{\WDPGP}{WDPGP}
\newcommand{\fTOPCA}{f$\perp$$\mathcal{T}$PCA}
\newcommand{\ftwoTRPCA}{fR$\mathcal{T}$PCA$(k)$}
\newcommand{\ftwoTWPCA}{fW$\mathcal{T}$PCA$(k)$}
\newcommand{\ftwoTDPCP}{f$\mathcal{T}$DPCP$(k)$}
\newcommand{\foneTRPCA}{fR$\mathcal{T}$PCA$(1,...,k)$}
\newcommand{\foneTWPCA}{fW$\mathcal{T}$PCA$(1,...,k)$}
\newcommand{\foneTDPCP}{f$\mathcal{T}$DPCP$(1,...,k)$}

\newtheorem{remark}{Remark}

\newtheorem{prop}{Proposition}
\newtheorem{dfn}{Definition}

\renewcommand{\paragraph}[1]{{\vspace{1mm}\noindent \bf #1}.}

\DeclareMathOperator*{\argmin}{arg\min}
\DeclareMathOperator*{\argmax}{arg\max}

\providecommand{\openbox}{\leavevmode
  \hbox to.77778em{%
  \hfil\vrule
  \vbox to.675em{\hrule width.6em\vfil\hrule}%
  \vrule\hfil}}
\makeatletter
\DeclareRobustCommand{\qed}{%
  \ifmmode
    \eqno \def\@badmath{$$}
    \let\eqno\relax \let\leqno\relax \let\veqno\relax
    \hbox{\openbox}%
  \else
    \leavevmode\unskip\penalty9999 \hbox{}\nobreak\hfill
    \quad\hbox{\openbox}%
  \fi
}

\crefname{eq}{eq}{eq}
\Crefname{Eq}{Eq}{Eq}
\crefname{thm}{theorem}{theorem}
\Crefname{Thm}{Theorem}{Theorem}
\crefname{prop}{Prop.}{Prop.}
\crefname{dfn}{Dfn.}{Dfn.}
\Crefname{Prop}{Proposition}{Proposition}
\crefname{remark}{remark}{remark}
\Crefname{Remark}{Remark}{Remark}
\Crefname{algorithm}{Alg.}{Alg.}

\crefname{table}{Tab.}{Tab.}
\Crefname{table}{Tab.}{Tab.}

%% file: content/intro.tex
\section{Introduction}\label{sec:intro}

Dimensionality reduction is at the heart of machine learning, statistics, and computer vision. Principal Component Analysis (PCA)~\cite{pearson1901liii,hotelling1933analysis} is a well-known technique for reducing the dimensionality of a dataset by linearly transforming the data into a new coordinate system where (most of) the variation in the data can be described with fewer dimensions than the initial data.
Thanks to its simplicity, effectiveness and versatility, PCA has quickly been extended to nonlinear transforms~\cite{scholkopf1997kernel,awate2014kernel,neumayer2020robust}, Riemannian manifolds~\cite{laparra2012nonlinearities,pennec2020advances,fletcher2004principal} or to unknown number of subspaces~\cite{vidal2005generalized}. These have proven indispensable for extracting meaningful information from complex datasets. 

In this paper, we present a unifying framework marrying a large family of PCA variants such as robust PCA~\cite{markopoulos2014optimal,polyak2017robust,kwak2013principal,lerman2018fast}, dual PCA~\cite{vidal2018dpcp}, PGA~\cite{sommer2010manifold} or tangent PCA~\cite{fletcher2004principal} specified by the norms and powers in a common objective function. Being able to accommodate all these different versions into the same framework allows us to innovate novel ones, for example, \emph{tangent dual PCA}, which poses a strong method for outlier filtering on manifolds. We further enrich the repertoire of available techniques by representing the space of eigenvectors as `flags'~\cite{alekseevsky1997flag}, a hierarchy of nested linear subspaces, in a vein similar to~\cite{pennec2018barycentric}. This `flagification' paves the way to a common computational basis, and we show how all these formulations can be efficiently implemented via a single algorithm that performs Riemannian optimization on Stiefel manifolds~\cite{boumal2023intromanifolds}.
This algorithm additionally contributes to the landscape of optimization techniques for dimensionality reduction, and we prove its convergence for the particular case of dual PCA.

\noindent In summary, \textbf{our contributions are}:
\begin{itemize}[itemsep=0.1pt,leftmargin=*,topsep=1pt]
    \item Generalization of PCA, PGA, and their robust versions leading to new novel variants of these principal directions
    \item A unifying flag manifold-based framework for computing principal directions of (non-)Euclidean data yielding novel (tangent) PCA formulations between $L_1$ and $L_2$ robust and dual principal directions controlled by flag types
    \item Novel weighting schemes, not only weighting the directions but also the subspaces composed of these directions.
    \item A practical way to optimize objectives on flags by mapping the problems into Stiefel-optimization, which removes the need to convey optimization on flag manifolds, an issue remaining open to date~\cite{ye2022optimization,szwagier2023rethinking}
\end{itemize}
Our theoretical exposition translates to excellent and remarkable findings, validating the usefulness of our novel toolkit in several applications, from outlier prediction to shape analysis. 
The implementation can be found \href{https://github.com/nmank/FunWithFlags}{here}.

%% file: content/relatedwork_short.tex
\section{Related Work}\label{sec:related}
Our work will heavily combine PCA with flag manifolds.

\paragraph{PCA and its variants}
Although there are many variants of PCA~\cite{hotelling1933analysis, scholkopf1997kernel, vidal2005generalized, candes2011robust, arenas2013kernel, zhang2018robust, huang2020communication, huang2022extension}, we focus on certain forms of Robust PCA (RPCA)~\cite{kwak2008principal, markopoulos2014optimal, neumayer2020robust} and Dual PCA (DPCA)~\cite{vidal2018dpcp}.  
As opposed to RPCA, DPCA finds directions orthogonal to RPCA and is
designed to work on datasets with outliers by minimizing a $0$-norm problem~\cite{vidal2018dpcp}.

We also consider the generalization of PCA to Riemannian manifolds, Principal Geodesic Analysis (PGA), which finds geodesic submanifolds that best represent the data~\cite{fletcher2004principal} and has been applied to the manifold of SPD matrices on ``real-world'' datasets~\cite{harandi2014manifold, smith2022data}. Though, exact PGA is hard to compute~\cite{sommer2010manifold, sommer2014optimization}.
Principal Curves~\cite{hastie1989principal} finds curves passing through the mean of a dataset which maximize variance and have seen their own enhancements~\cite{laparra2012nonlinearities}. 
Linearized (tangent) versions of PGA perform PCA on the tangent space to the mean of the data~\cite{abboud2020robust, awate2014kernel}.
Geodesic PCA (GPCA) removes the mean requirement \cite{huckemann2006principal, huckemann2010intrinsic}. 
Barycentric Subspace Analysis (BSA) realizes PCA on more classes of manifolds than just Riemannian manifolds by generalizing geodesic subspaces using weighted means of reference points~\cite{pennec2018barycentric}.
Recent years have witnessed further generalizations~\cite{pennec2020advances,rabenoro2022geometric,buet2023flagfolds} as we will mention later.

\paragraph{Flag manifolds}
Flag manifolds are useful mathematical objects~\cite{wiggerman1998fundamental,donagi2008glsms,alekseevsky1997flag,kirby2001geometric, ye2022optimization}. Nishimori~\etal use Riemannian optimization and flag manifolds to formulate variations of independent component analysis~\cite{nishimori2006riemannian0,nishimori2006riemannian1,nishimori2007flag, nishimori2008natural,nishimori2007flag}. Others represent an average subspace as a flag~\cite{draper2014flag,mankovich2022flag,mankovich2023subspace} and find average flags~\cite{Mankovich_2023_ICCV}. Flags even arise as nested principal directions and in
manifold variants of PCA \cite{pennec2018barycentric,pennec2020advances,ye2022optimization}.

%% file: content/background.tex
\vspace{-1mm}\section{Preliminaries}\vspace{-1mm}\label{sec:bg}
Let us start by briefly introducing Riemannian geometry, flag manifolds, and methods for finding principal directions. The flag and flag manifold definitions follow~\cite{Mankovich_2023_ICCV}.
\begin{dfn}[Riemannian manifold~\cite{lee2006riemannian}]
A \emph{Riemannian manifold} $\Man$ is a smooth manifold with a positive definite inner product $\langle \cdot, \cdot \rangle: \TxM \times \TxM \to \R$ defined on the tangent space $\TxM$ at $\x \in \Man$.
\end{dfn}
\begin{dfn}[Geodesics \& Exp/Log-maps~\cite{lee2006riemannian}]
A \emph{geodesic} $\gamma : \R \to \Man$ parameterizes a path from $\gamma(0) = \x$ to $\gamma(1) = \y$. 
The \emph{exponential map} $\Exp_{\x}(\bv) : \TxM \to \Man$ maps a vector $\bv \in \TxM$ to the manifold in a length preserving fashion such that $\dot\gamma_{\bv}(0)=\bv$ and $\Exp_{\x}(\bv)=\y$. 
Its inverse, the \emph{logarithmic map} is $\Log_{\x}(\y) : \Man \to \TxM$ for $\x,\y \in \Man$ and computes the tangent direction from $\x$ to $\y$. Hence,  $\gamma(t) = \Exp_{\x}\left({\tau\Log_{\x}(\y)}\right)$ for $\tau \in \R$ is the geodesic curve. $\SubMan \in \Man$ is said to be a geodesic submanifold at $\x \in \Man$ if all geodesics through $\x$ in $\SubMan$ are geodesics in $\Man$.
\end{dfn}

\begin{dfn}[Flag]\label{def:flagobj}
A \emph{flag} is a nested sequence of subspaces of a finite-dimensional vector space $\mathcal{V}$ of \emph{increasing} dimension, is the filtration $\{\emptyset \}=\mathcal{V}_0 \subset \mathcal{V}_1 \subset \ldots \subset \mathcal{V}_{k} \subset \mathcal{V}$ with $0=n_0<n_1<\ldots<n_k<n$ where $\mathrm{dim} \mathcal{V}_i=n_i$ and $\mathrm{dim} \mathcal{V}=n$.  The \emph{type} or \emph{signature} of this flag is $(n_1, \ldots ,n_k;n)$ or $(n_1, \ldots ,n_k)$.

\noindent\textbf{\emph{Notation}:} We denote a flag using a point $\X$ on the Stiefel manifold~\cite{edelman1998geometry} $St(n_k, n) = \{\X \in \R^{n \times n_k} : \X^T \X = \I \}$. Given $\X \in St(n_k, n)$ and $\x_i$, $i^{\text{th}}$ column of $\X$, we define 
\begin{equation}\label{eq: Xi def}
\X_{i+1} = \begin{bmatrix} \x_{n_{i}+1} & \x_{n_{i}+2} & \cdots &  \x_{n_{i+1}} \end{bmatrix}  \in \R^{n \times m_{i+1}}.
\end{equation}
where $m_i = n_i-n_{i-1}$ for $i=1,2,\ldots, k$. $[\X_1, \ldots, \X_i]$ denotes the span of the columns of $\{ \X_1, \ldots, \X_i\}$. Then
\begin{equation}
     [ \X_1 ]  \subset [\X_1, \X_2] \subset \cdots \subset  [\X_1, \ldots, \X_k] = [\X] \subset \R^n. \nonumber
\end{equation}
is a flag of type $(n_1, \ldots , n_k;n)$ and is denoted $[\![\X]\!]$.
\end{dfn}

\begin{dfn}[Flag manifold]
The set of all flags of type $(n_1, \ldots, n_k; n)$ is called the flag manifold due to its manifold structure. We refer to this \emph{flag manifold} as $\flag(n_1, \ldots, n_k; n)$ or $\flag(n+1)$. Flags generalize Grassmann and Stiefel manifolds~\cite{edelman1998geometry} because $\flag(n_k; n)= Gr(n_k,n)$ and $\flag(1,\dots,n_k; n)=St(n_k,n)$. 
We denote flags as $\flag(n+1)$ using the fact from~\cite{ye2022optimization}:
\begin{equation*}
\flag(n+1) = St(n_k, n)/O(m_1) \times O(m_2) \times \cdots \times O(m_{k+1}).
\end{equation*}
\end{dfn}

%% file: content/pca.tex
\section{Generalizing PCA and Its Robust Variants}
\emph{Principal directions} are the directions where the data varies. We now review and go beyond the celebrated principal component analysis (PCA)~\cite{hotelling1933analysis} algorithm and its variants. In what follows, we present PCA in a generalizing framework, which further yields novel variants.
We consider a set of $p$ centered samples (points with a sample mean of $0$) with $n$ random variables (features) $\mathcal{X} = \{ \x_j \subset \R^n\}_{j=1}^p $ and collect these data in the matrix $\X = [\x_1,\x_2,\ldots,\x_p]$.
\begin{dfn}[PCA~\cite{hotelling1933analysis}]\label{dfn:pca}
PCA aims to linearly transform the data into a new coordinate system, specified by a set of $k < n$ orthonormal vectors $\{\bu_i \in \R^n\}_{j=1}^k$, where (most of) the variation in the data can be described with fewer dimensions than the initial data. 
The $i^\mathrm{th}$ principal direction is obtained either by maximizing variance \cref{eq: maxpca} or minimizing reconstruction error (Eq.~\eqref{eq: minpca}):
\begin{align}\label{eq: maxpca}
\mathbf{u}_{i} &= \argmax_{\substack{\mathbf{u}^T \mathbf{u} = 1, \,\mathbf{u} \in S_i^\perp }} \,\,\,\mathbb{E}_j \left[\|  \pi_{S_{\mathbf{u}}} (\x_j) \|_2^2 \right]\\
\mathbf{u}_{i} &= \argmin_{\substack{\mathbf{u}^T \mathbf{u} = 1,\, \mathbf{u} \in S_{i-1}^\perp}} \mathbb{E}_j \left[\|  \bm{\x_j} - \pi_{S_{\mathbf{u}}} (\x_j) \|_2^2 \right],\label{eq: minpca}
\end{align}
where $\pi_{S_\bu}(\x) := \bu \bu^T \x$, $S_i = \mathrm{span}\{ \bu_1, \bu_2, \ldots, \bu_i \}$, $S_i^\perp$ denotes its orthogonal complement, and $\mathbb{E}_j$ denotes expectation over $j$.
Although these objectives can be achieved jointly~\cite{wang2023max}, in this work, we focus on the more common practice specified above.
Notice that~\cref{eq: maxpca} is equivalent (up to rotation) to solving the following optimization:
\begin{equation}\label{eq: pca_opt_rot}
    \argmax_{\U^T\U = \I} \tr(\U^T \X \X^T \U)
\end{equation}
\end{dfn}
Directions of maximum variance captured by the naive PCA are known to be susceptible to outliers in data. This motivated a body of work devising more robust versions. 
\begin{dfn}[Generalized PCA]\label{dfn:pcagen}
The formulations in~\cref{dfn:pca} can be generalized by using arbitrary $L_p$-norms, $q$-powers, and weighted with real weights $\{w_1,w_2,\ldots, w_p\}$:
\begin{align}
\U^\star&=\argmax_{\mathbf{U}^T \mathbf{U} = \I} \,\,\mathbb{E}_j \left[ w_j\|  \pi_{S_{\mathbf{U}}} (\x_j) \|_p^q \right],\label{eq: maxpca_gen}\\
\U^\star&=\argmin_{\mathbf{U}^T \mathbf{U} = \I } \,\,\,\mathbb{E}_j \left[w_j\|  \bm{\x_j} - \pi_{S_{\mathbf{U}}} (\x_j) \|_p^q \right],\label{eq: minpca_gen}
\end{align}
where $\pi_{S_\U}(\x) := \U \U^T \x$. Both problems recover PCA up to a rotation when $p=q=2$. When $p \leq 2$ and $ q < 2$, a more \emph{outlier robust} version of PCA is achieved, and the two formulations become different. The variance-maximizing (\cref{eq: maxpca_gen}) $q=1$ case is known as \LpRPCA~\cite{kwak2013principal}. Specifically, when $q=p=1$, it recovers \LoneRPCA~\cite{markopoulos2014optimal} and when $q=1$, $p=2$, it recovers \LtwoRPCA~\cite{polyak2017robust}. 
On the other hand, minimizing the reconstruction error  (\cref{eq: minpca_gen}) leads to $L_1$-Weiszfeld PCA (\LoneWPCA)~\cite{neumayer2020robust} for $q=p=1$ and \LtwoWPCA~\cite{ding2006r} for $q=1$, $p=2$.
\end{dfn}

\paragraph{Generalizing Dual-PCA}
Momentarily assume that the data matrix can be decomposed into inliers $\X_I$ and outliers $\X_O$: $\X =  [\X_I, \X_O] \mathbf{P}$, where $\mathbf{P}$ is a permutation matrix. 
\begin{dfn}[Dual-PCA~\cite{vidal2018dpcp}]\label{dfn:DPCA}
    Assuming the samples live on the unit sphere, $\{\x_i\}_{i=1}^p \in \mathbb{S}^{n-1}$, DPCA seeks a subspace $S_\ast = \mathrm{span}\{ \vb_1,  \vb_2, \ldots,  \vb_k\}$ so that $S_\ast^\perp$ contains the most inliers (e.g., columns of $\X_I$). In other words, we seek vectors $\{\vb_i\}_{i=1}^k$ that are \emph{as orthogonal as possible} to the span of the inliers. This is the iterative optimization
    \begin{equation}~\label{eq: dualpca}
        \vb_{i} =  \argmin_{\| \vb \| = 1, \: \vb \in S_{i-1}^\perp} \| \X^T \vb \|_0.
    \end{equation}    
    When we re-write the maximum variance formulation of the \PCA~optimization in~\cref{eq: maxpca_gen} as the minimization
    (e.g., $\vb_{i} =  \argmin_{\| \vb \| = 1, \: \bu \in S_{i-1}^\perp} \| \X^T \vb \|_2^2$) we see that Dual-\PCA~minimizes a similar objective that is robustified by considering a $0$-norm.
\end{dfn}
\begin{dfn}[Dual Principal Component Pursuit (DPCP)]\label{dfn:DPCP}
Relaxing the $L_0$-norm sparse problem into an $L_1$-norm one turns the DPCA problem into \LoneDPCP~\cite{vidal2018dpcp}. When solved via an $L_2$-relaxed scheme, we recover \LtwoDPCP~(called DPCP-IRLS~\cite{vidal2018dpcp} and equivalent to the (spherical) Fast Median Subspace~\cite{lerman2018fast}). 
\end{dfn}
\begin{dfn}[Dual PCA Generalizations]\label{dfn: genDualPCA}
In general, we can think of~\cref{dfn:DPCA,dfn:DPCP} as variance-minimizing (\cref{eq: mindpca_gen}) and reconstruction-error-maximizing (\cref{eq: maxdpca_gen}), respectively: 
\begin{align}\label{eq: mindpca_gen}
\B^\star&=\argmin_{\B^T \B = \I} \,\mathbb{E}_j \left[\|  \pi_{S_{\B}} (\x_j) \|_p^q \right]\\
\B^\star&=\argmax_{\B^T \B = \I } \mathbb{E}_j \left[\|  \bm{\x_j} - \pi_{S_{\B}} (\x_j) \|_p^q \right].\label{eq: maxdpca_gen}
\end{align}
Using~\cref{eq: mindpca_gen}, we recover \LoneDPCP~with $p=q=1$ and \LtwoDPCP~with $p=2$ and $q=1$.
\end{dfn}

\begin{remark}[New DPCP Variants]
    Two other possibilities optimize~\cref{eq: maxdpca_gen} when $p=q=1$ and $p=2,q=1$. We call these methods $L_p$-Weiszfeld DPCPs (WDPCPs), specifically, \LoneWDPCP~and \LtwoWDPCP~for different $p$ values.
\end{remark}

\paragraph{Extension to Riemannian manifolds}
Principal geodesic analysis (PGA)~\cite{fletcher2004principal} generalizes 
PCA for describing the variability of data $\{\x_i\in\Man\}_{i=1}^p$ on a Riemannian manifold $\Man$, induced by the \emph{geodesic distance} $d(\cdot, \cdot):\Man\times\Man\to\R_+$. To this end, PGA first requires a manifold-mean:
\begin{equation}\label{eq: manifold_mean}
    \bmu = \argmin_{\y \in \Man} \mathbb{E}_j \left[ d(\x_j, \y)^2 \right],
\end{equation}
where the local minimizer is called the \emph{Karcher mean} and if there is a global minimizer, it is called the \emph{Fr\'echet mean}. A more robust version of the Karcher mean is the \emph{Karcher median} which minimizes $\mathbb{E}_j \left[ d(\x_j, \y) \right]$ and can be estimated by running a Weiszfeld-type algorithm~\cite{aftab2014generalized}.
Next, PGA 
uses geodesics rather than lines, locally the shortest path between two points, as one-dimensional subspaces.
\begin{dfn}[PGA~\cite{fletcher2003statistics,sommer2010manifold}]
The $i^\mathrm{th}$ principal geodesic for (exact) PGA is defined as $\gamma_i(t) = \Exp_{\bmu}(\bu_i t)$
constructed either by maximizing variance (\cref{eq:pga_max}) or by minimizing reconstruction error (or unexplained variance) (\cref{eq:pga_min}):
\begin{align}
\bu_{i} &= \argmax_{ \| \bu \| = 1,\, \: \bu \in S_{i-1}^\perp} \mathbb{E}_j \left[ d( \bmu, \pi_{\SubMan(S_\bu)} (\x_j))^2 \right]\label{eq:pga_max}\\
\bu_{i} &= \argmin_{ \| \bu \| = 1,\, \: \bu \in S_{i-1}^\perp} \mathbb{E}_j \left[ d(\x_j, \pi_{\SubMan(S_\bu)} (\x_j))^2 \right],\label{eq:pga_min} 
\end{align}
where $S_\bu = \mathrm{span}\{ \bu\}$, the $i^\mathrm{th}$ subspace of $\TmuM$, and
\begin{equation}
    S_i = \mathrm{span}\{ \bu_1, \bu_2, \ldots, \bu_i \}.
\end{equation}
The projection operator onto $\SubMan(S) \subset \Man$, the geodesic submanifold of $S \subset \TmuM$, is \footnote{$\pi_{H}$ can use any monotonically increasing function of distance on $\Man$.}
\begin{equation}\label{eq:projdef}
\pi_{\SubMan}(\x) = \argmin_{\z \in \SubMan } d(\z, \x).
\end{equation}
where $\SubMan(S) = \{ \Exp_{\bmu}(\bv) : \bv \in S\}$.
\end{dfn}
\begin{remark}
    In contrast to PCA, PGA needs to explicitly define $\SubMan(S)$ because the tangent space and manifold are distinct. While in Euclidean space, maximizing variances is equivalent to minimizing residuals,~\cref{eq:pga_max} and~\cref{eq:pga_min} are not equivalent on Riemannian manifolds~\cite{sommer2010manifold}. PGA results in a flag of subspaces of the tangent space of type $(1,2,\ldots,k; \text{\emph{dim}}(T_{\bmu}(\Man))$ in~\cref{eq: tangent_flag} along with an increasing sequence of geodesic submanifolds in~\cref{eq:subman_flag}
    \begin{align}
        &S_1 \subset S_2 \subset \cdots \subset S_k \subset T_{\bmu}(\Man),\label{eq: tangent_flag}\\
        &\SubMan(S_1) \subset \SubMan(S_2) \subset \cdots \subset \SubMan(S_k) \subset \Man.\label{eq:subman_flag}
    \end{align} 
    Additionally, a set of principal directions, $\{\bu_i\}_{i=1}^k$, form a (totally) geodesic submanifold $\SubMan(S_k) \subseteq \Man$ as long as geodesics in $\SubMan(S_k)$ are carried to geodesics in $\Man$~\cite{tabaghi2023principal}.
\end{remark}

\noindent Similar to~\cref{dfn:pcagen}, we now generalize PGA.
\begin{dfn}[PGA Generalizations]\label{dfn:pgagen}
Let $\{w_j\}_{j=1}^p \subset \R$ denote a set of weights. The weighted principle geodesic is $\gamma_i(t) = \Exp_{\bmu}(\bu_it)$ where $\bu_i$ maximizes / minimizes:
\begin{align}
\bu_{i} &= \argmax_{\substack{ \| \bu \| = 1,\, \bu \in S_{i-1}^\perp}} \mathbb{E}_j \left[ w_jd( \bmu, \pi_{\SubMan(S_\bu)} (\x_j))^q \right]\label{eq:wt_pga_max}\\
\bu_{i} &= \argmin_{\substack{ \| \bu \| = 1,\, \bu \in S_{i-1}^\perp}} \mathbb{E}_j \left[ w_jd(\x_j, \pi_{\SubMan(S_\bu)} (\x_j))^q \right],\label{eq:wt_pga_min} 
\end{align}
This recovers PGA when $q=2$ and $w_j = 1$ for all $j$.
\end{dfn}

\begin{remark}[Tangent-PCA (\TPCA)~\cite{fletcher2004principal}]\label{rem:TPCA}
    PGA is known to be computationally expensive to compute except on a few simple manifolds~\cite{tabaghi2023principal,said2007exact}. 
    As a remedy, Fletcher~\etal~\cite{fletcher2004principal} leverage the Euclidean-ness of the tangent space to define principal geodesics as $\gamma(t) = \Exp_{\bmu}(t\bu_i)$ where $\{\bu_i\}_{i=1}^k$ are the principal components of $\{\Log_{\bmu} \x_j \}_{j=1}^p$. This approximation, known as \emph{Tangent-PCA}, and we will later
    approximately invert it to reconstruct data on $\Man$ by (i) using principal directions to reconstruct the data on $T_{\bmu}(\Man)$, then (ii) mapping the reconstruction to $\Man$ using $\Exp_{\bmu}(\cdot)$. 
\end{remark}

\begin{prop}[Robust PGAs (RPGA \& WPGA)]
    Setting $1\leq q<2$, gives us novel, robust formulations of the PGA problem (RPGA and WPGA) defined in~\cref{dfn:pgagen}, which we will solve in the unifying flag framework we provide.     
    While general robust manifold-optimizers such as \emph{robust median-of-means}~\cite{lin2020robust} can be used to implement RPGA and WPGA, to be consistent with \TPCA, we will approximate these problems by performing RPCA and WPCA in the tangent space of the robust Karcher median (removing the square in~\cref{eq: manifold_mean}).
    We will refer to these tangent space versions as tangent \RPCA~ (\TRPCA) and tangent \WPCA~ (\TWPCA). 
\end{prop}

\noindent We are now ready to formulate novel, dual versions of PGA.
\begin{prop}[Dual PGA (DPGA)]
    Given a dataset on a Riemannian manifold, we define dual robust principal directions, analogous to DPCA (\cref{dfn: genDualPCA}):
    \begin{align}\label{eq:dpga_min}
\vb_{i} &= \argmin_{\substack{ \| \vb \| = 1 \,,\, \vb \in S_{i-1}^\perp}} \mathbb{E}_j \left[ w_j d( \bmu, \pi_{H(S_\vb)} (\x_j))^q \right]\\
\vb_{i} &= \argmax_{\substack{ \| \vb \| = 1 \,,\, \vb \in S_{i-1}^\perp}} \mathbb{E}_j \left[ w_j d( \x_j, \pi_{H(S_\vb)} (\x_j))^q \right].\label{eq:dpga_max} 
\end{align}
We refer to these novel principal directions as \DPGP~ (\cref{eq:dpga_min}) and \WDPGP~ (\cref{eq:dpga_max}). Again, we can approximate these problems by performing DPCP and WDPCP in the tangent space, resulting in the tractable algorithms of tangent DPCP (\TDPCP) and tangent WDPCP (\TWDPCP).
\end{prop}

\begin{remark}[Normalization]
    Classical DPCA works with datasets normalized to the unit sphere. Our tangent formulations of DPCP variations do not perform this preprocessing on the tangent space.
\end{remark}
{We summarize all the PCA methods as well as our extensions in~\cref{tab:pca_summary}.}

%% file: content/method.tex
\begin{table}[t]
    \centering
    \begin{tabular}{@{\:}c@{\:}|@{\:}c@{\:}|@{\:}c@{\:}||@{\:}c@{\:}c@{\:}c@{\:}}
        & & $(p,q)$ & Variance & Rec. Err. & $\flag(\cdot;n)$\\
        \toprule
        \multirow{6}{*}{\rotatebox[origin=c]{90}{Euclidean}} & \multirow{3}{*}{\rotatebox[origin=c]{90}{Primal}} & $(2,2)$ & \PCA~\cite{hotelling1933analysis} & \PCA~\cite{hotelling1933analysis} & $(1,2,...,k)$\\
        & & $(2,1)$ &  \LtwoRPCA ~\cite{polyak2017robust} & \LtwoWPCA~\cite{ding2006r} & $(k)$\\
        & & $(1,1)$ & \LoneRPCA~\cite{markopoulos2014optimal} & \LoneWPCA~\cite{neumayer2020robust} & $(1,2,...,k)$\\
        \cmidrule{2-6}
        & \multirow{3}{*}{\rotatebox[origin=c]{90}{Dual}} & $(2,2)$ & \OPCA~\cite{vidal2018dpcp} & \OPCA~\cite{vidal2018dpcp} & $(1,2,...,k)$\\
        & & $(2,1)$ &  \LtwoDPCP~\cite{vidal2018dpcp} & \textcolor{blue}{\LtwoWDPCP} & $(k)$\\
        & & $(1,1)$ & \LoneDPCP~\cite{vidal2018dpcp} & \textcolor{blue}{\LoneWDPCP} & $(1,2,...,k)$\\
        \midrule
        \multirow{6}{*}{\rotatebox[origin=c]{90}{Manifold}} & \multirow{3}{*}{\rotatebox[origin=c]{90}{Primal}} & $(2,2)$ & \TPCA~\cite{fletcher2004principal} & \TPCA~\cite{fletcher2004principal} & $(1,2,...,k)$\\
        & & $(2,1)$ &  \textcolor{blue}{\LtwoTRPCA} & \textcolor{blue}{\LtwoTWPCA} & $(k)$\\
        & & $(1,1)$ & \textcolor{blue}{\LoneTRPCA} & \textcolor{blue}{\LoneTWPCA} & $(1,2,...,k)$\\
        \cmidrule{2-6}
        & \multirow{3}{*}{\rotatebox[origin=c]{90}{Dual}} & $(2,2)$ & \textcolor{blue}{\OTPCA} & \textcolor{blue}{\OTPCA} & $(1,2,...,k)$\\
        & & $(2,1)$ &  \textcolor{blue}{\LtwoTDPCP} & \textcolor{blue}{\LtwoTWDPCP} & $(k)$\\
        & & $(1,1)$ & \textcolor{blue}{\LoneTDPCP} & \textcolor{blue}{\LoneTWDPCP} & $(1,2,...,k)$\\
        \bottomrule
    \end{tabular}
    \caption{A summary of variants of \PCA, robust \PCA~and tangent \PCA. The new \PCA~variants introduced in this paper are highlighted in \textcolor{blue}{blue}. For robust variants of \PCA: optimizing over $\flag(1,2,\dots,k;n) = St(k,n)$ recovers $L_1$ and optimizing over $\flag(k;n) = Gr(k,n)$ recovers $L_2$ formulations. Optimizing for any other flag type will provide a collection of novel algorithms between $L_1$ and $L_2$ versions.}
    \label{tab:pca_summary}
\end{table}

\section{Flagifying PCA and Its Robust Variants}\vspace{-1mm}\label{sec:method}
We now re-interpret PCA in Euclidean spaces as an optimization on flags of linear subspaces. This flagification will later enable us to introduce more variants and algorithms.

\begin{dfn}[Flagified (weighted-)PCA (fPCA)~\cite{pennec2018barycentric}]
A (weighted-)flag of principal components is the solution to:
    \begin{equation}\label{eq:wflagpca}
         [\![\U]\!]^\star=\argmax_{[\![\U]\!] \in \flag(n+1)}  \mathbb{E}_j \left[ \sum_{i=1}^k w_{ij}\| \pi_{\U_i} (\x_j)\|_2^2 \right].
    \end{equation}
    where $w_{ij}$ denote the weights.
    We refer to a weighted flag PCA algorithm optimized over $\flag(n_1,n_2,\dots,n_k;n)$ as $\mathrm{weighted-fPCA}(n_1,n_2,\dots,n_k;n)$. When $w_{ij}=1 \,\forall i,j$, we recover $\mathrm{fPCA}$.
\end{dfn}

\begin{remark}
The solution to~\cref{eq: pca_opt_rot} is only unique up to rotation, and PCA is unique (up to column signs) because we order by eigenvalues. This ordering imposes a flag structure. Interpreting this optimization problem over a flag emphasizes the nested structure of principal subspaces~\cite{damon2014backwards} and provides a slight loosening of the strict eigenvalue ordering scheme from PCA. Also note that the joint optimization over the whole flag of subspaces (instead of optimizing each subspace independently) poses a computational challenge, preventing~\cite{pennec2018barycentric} from a practical implementation. This gap is filled via a manifold optimization in~\cite{ye2022optimization} by characterizing the Riemannian geometry of $\flag(\cdot)$. We provide further details in the supplementary.
\end{remark}

\noindent Building off of these, we now flagify the robust variants of PCA before introducing new dimensionality reduction algorithms and moving onto the principal geodesic.

\paragraph{Flagified Robust (Dual-)PCA variants}
To respect the nested structure of flags, we must embed the flag structure into the optimization problem. Generalized versions of robust PCA in~\cref{eq: flagRPCAmax} and Dual PCA in~\cref{eq: flagDPCAmin} change the objective function value and the space over which we optimize. 
We state these flagified formulations below.
\begin{dfn}[Flagified (Dual-)PCA]
In the sequel, we define flagified (f) \RPCA~/ \WPCA~/ \DPCP~/ \WDPCP:
\begin{align}\label{eq: flagRPCAmax}
&[\![\U]\!]^\star = \\
&\begin{cases}
        \argmax\limits_{[\![\U]\!] \in \flag(n+1)}  \mathbb{E}_j \left[ \sum_{i=1}^k\| \pi_{\U_i} (\x_j)\|_2 \right], & \mathrm{(fRPCA)}\\
        \argmin\limits_{[\![\U]\!] \in \flag(n+1)} \mathbb{E}_j \left[ \sum_{i=1}^k \| \x_j  -\pi_{\U_i} (\x_j)\|_2 \right], & \mathrm{(fWPCA)}
        \end{cases}\nonumber\\
&\nonumber\\
&[\![\B]\!]^\star = \label{eq: flagDPCAmin}\\
&\begin{cases}
        \argmin\limits_{[\![\B]\!] \in \flag(n+1)} \mathbb{E}_j \left[ \sum_{i=1}^k\| \pi_{\B_i} (\x_j)\|_2 \right], & \mathrm{(fDPCP)}\\
        \argmax\limits_{[\![\B]\!] \in \flag(n+1)} \mathbb{E}_j \left[ \sum_{i=1}^k \| \x_j  -\pi_{\B_i} (\x_j)\|_2 \right], & \mathrm{(fWDPCP)}
        \end{cases}\nonumber
\end{align}
where $\U_i$ and $\B_i$ as $\X_i$ is defined using~\cref{eq: Xi def}.
\end{dfn}
\begin{remark}
    Formulating these flagified robust PCAs over $\flag(1,2,\dots,k;n)$ recovers $L_1$ formulations and over $\flag(k;n)$ recovers $L_2$ of robust PCA and \DPCP~formulations. This fact is enforced in~\cref{tab:pca_summary}. 
\end{remark}
Inspired by Mankovich and Birdal~\cite{Mankovich_2023_ICCV}, we now show how to implement these robust variants by showing equivalent optimization problems on the Stiefel manifold~\cite{edelman1998geometry}.
We start by viewing weighted \fPCA~in~\cref{eq:wflagpca} as a Stiefel optimization problem in~\cref{prop:wfPCA}. For the rest of this section, we will slightly abuse notation and use $[\![\U ] \!]$ for flags of both primal  
 and dual principal directions, discarding $\B$. We will provide the necessary proofs in our supplementary material.
\begin{prop}[Stiefel optimization of (weighted) fPCA]\label{prop:wfPCA}
    Suppose we have weights $\{w_{ij}\}_{i=1,j=1}^{i=k,j=p}$ for a dataset $\{ \x_j \}_{j=1}^p \subset \R^n$ along with a flag type $(n_1,n_2,\dots, n_k; n)$. We store the weights in the diagonal weight matrices $\{\W_i\}_{i=1}^k$ with diagonals $(\W_i)_{jj} = w_{ij}$. If
    \begin{equation}\label{eq:wflag_pca}
        \U^\ast  = \argmax_{\U \in St(n_k,n)} \sum_{i=1}^k \tr \left( \U^T \X \W_i \X^T \U \I_i \right)
    \end{equation}
    where $\I_i$ is determined as a function of the flag signature. For example, for $\flag(n+1)$: 
    \begin{equation}\label{eq: Ii}
    (\I_i)_{l,s} = 
    \begin{cases}
        1, & l = s \in \{ n_{i-1} + 1, 
 n_{i-1} + 2, \dots, n_i\} \\
        0, &\mathrm{ otherwise}\nonumber\\
    \end{cases}    
\end{equation}
    Then $[\![\U^\ast ]\!] = [\![\U ]\!]^\ast$ is the weighted fPCA of the data with the given weights (e.g., solves~\cref{eq:wflagpca}) as long as we restrict ourselves to a region on $\flag(n+1)$ and $St(n_k,n)$ where (weighted)~\fPCA~is convex.
\end{prop}
\begin{proof}[Sketch of the proof]
Our proof, whose details are in the supp. material, closely follows~\cite{Mankovich_2023_ICCV}.
\end{proof}
We propose an algorithm for finding (weighted)~\fPCA~using Stiefel Conjugate Gradient Descent (Stiefel-CGD)~\cite{hager2006survey,sato2022riemannian} in the supplementary.

Next, we translate the flagified robust PCA optimizations in~\cref{eq: flagRPCAmax,eq: flagDPCAmin} to problems over the Stiefel manifold with diagonal weight matrices
    \begin{align}\label{eq: max_weights}
    (\mathbf{W}_i^+([\![\U]\!]))_{jj}  &= \max \left\{ \| \U \I_i \U^T\x_j\|_2, \epsilon \right\}^{-1},\\
    (\mathbf{W}_i^-([\![\U]\!]))_{jj} &= \max \left\{ \|\x_j -\U \I_i \U^T\x_j\|_2, \epsilon \right\}^{-1},\label{eq: min_weights}
\end{align}
chosen according to the robust \fPCA~optimization of concern, as outlined in~\cref{tab:weight_org}.
\begin{table}[t]
    \centering
    \begin{tabular}{c||c|c}
        \PCA~Variant & \fRPCA~/ \fDPCP & \fWPCA~/ \fWDPCP\\
        \midrule
        Weight & $\mathbf{W}_i^+$ from~\cref{eq: max_weights} & $\mathbf{W}_i^-$ from~\cref{eq: min_weights}
    \end{tabular}
    \caption{Weight matrix assignment according to the flagified robust PCA formulation.\vspace{-3mm}}
    \label{tab:weight_org}
\end{table}

\begin{prop}[Stiefel optimization for flagified Robust (Dual-)PCAs]\label{prop:fPCA_all}
   We can formulate \fRPCA, \fWPCA, \fDPCP, and \fWDPCP~as optimization problems over the Stiefel manifold using $[\![ \U ] \!]^\ast = [\![\U^\ast] \!]$ and the following: 
    \begin{align}\label{eq: stRPCAmax}
    &\U^\star = \\
    &\begin{cases}
            \argmax\limits_{\U \in St(n, n_k)}  \sum_{i=1}^k \tr \left(  \U^T \Pm_i^+\U \I_i  \right), & \mathrm{(fRPCA)}\\
            \argmin\limits_{\U \in St(n, n_k)}\sum_{i=1}^k \tr \left(\Pm_i^- -  \U^T \Pm_i^-\U \I_i  \right), & \mathrm{(fWPCA)}
            \end{cases}\nonumber\\
    &\nonumber\\
    &\U^\star = \label{eq: stDPCAmin}\\
    &\begin{cases}
            \argmin\limits_{\U \in St(n, n_k)} \sum_{i=1}^k \tr \left(  \U^T \Pm_i^+ \U \I_i  \right), & \mathrm{(fDPCP)}\\
            \argmax\limits_{\U \in St(n, n_k)} \sum_{i=1}^k\tr \left(\Pm_i^- -  \U^T \Pm_i^-\U \I_i  \right) & \mathrm{(fWDPCP)}
            \end{cases}\nonumber
    \end{align}
    where $\Pm^-=\X \W^-_i([\![\U]\!]) \X^T$, $\Pm^+=\X \W^+_i([\![\U]\!]) \X^T$ and $\W^-_i([\![\U]\!])$, $\W^+_i([\![\U]\!])$ are defined in~\cref{tab:weight_org} as long as we restrict ourselves to a region on $\flag(n+1)$ and $St(n_k,n)$ where flag robust and dual PCAs are convex.
\end{prop}

~\cref{eq: stRPCAmax,eq: stDPCAmin} offer natural iterative re-weighted optimization schemes on the Stiefel manifold for obtaining flagified robust PCA variants, where we calculate a weighted flagified PCA at each iteration with weights defined in~\cref{tab:weight_org}. This is similar to~\cite{busam2017camera}. We summarize these algorithms in~\cref{alg: flagified pca}.
We further establish the convergence guarantee for the case of \fDPCP~\cref{cor: fdpcp converges} and leave other convergence results to future work. The assumption in our convergence guarantee is realistic because in the presence of real-world, noisy data, we cannot expect to recover dual principal directions that are perfectly orthogonal to the inlier data points.  We leave dropping this assumption, leveraging optimizing our algorithm, and more advanced proof techniques similar to those in~\cite{aftab2015convergence, peng2023convergence} for a future study.

\begin{prop}[Convergence of \cref{alg: flagified pca} for \fDPCP]\label{cor: fdpcp converges}
    \cref{alg: flagified pca} for \fDPCP~converges as long as $\| \U \I_i \U^T \x_j\|_2 \geq \epsilon$ $\forall i,j$ as long as we restrict ourselves to a region on $\flag(n+1)$ and $St(n_k,n)$ where \fDPCP~is convex.
\end{prop}
\begin{proof}[Sketch of the proof]
 Similar to~\cite{beck2015weiszfeld,Mankovich_2023_ICCV,mankovich2023subspace}, we first show that an iteration of~\cref{alg: flagified pca} decreases the \fDPCP~objective value and then convergence follows easily.
\end{proof}

\begin{algorithm}[t]
\setstretch{1.13}
\caption{\fRPCA, \fWPCA, \fDPCP}\label{alg: flagified pca}
 \textbf{Input}: {Data $\{\x_j \in \R^n\}_{j=1}^p $, flag type $(n+1)$, $\epsilon>0$}\\
 \textbf{Output}: Flagified principal directions $[\![\U]\!]^\ast $ \\[0.25em] 
 Initialize $[\![\U]\!]$\\
 \While{(not converged)}
 {
     \vspace{.1cm}
    \textbf{Assign weights:}\\
    \Case{fRPCA or fDPCP}{
        Assign $\{\mathbf{W}_i^+([\![\U]\!])\}_{i=1}^k$ using~\cref{eq: max_weights}
            }
    \Case{fWPCA}{
        Assign $\{\mathbf{W}_i^-([\![\U]\!])\}_{i=1}^k$  using~\cref{eq: min_weights}
            }
    \vspace{.1cm}
    \textbf{Update estimate:}\\
    \Case{fRPCA} {
            $\mathbf{A}^+ \gets \sum_{i=1}^k  \I_i \U^T \X \W_i^+ \X^T$ \\
            $\U \gets \argmax_{\Z \in St(k,n)} \mathbf{A}^+ ( [\![ \U ]\!] ) \Z$
           }
    \Case{fWPCA}{
            $\mathbf{A}^- \gets \sum_{i=1}^k  \I_i \U^T \X \W_i^- \X^T$ \\
            $\U \gets \argmin_{\Z \in St(k,n)} \mathbf{A}^- ( [\![ \U ]\!] ) \Z$
            }
    \Case{fDPCP}{
        Assign $\U$ using~\cref{eq: stDPCAmin} with $\{\mathbf{W}_i^+([\![\U]\!])\}$
           }
}
$[\![\U]\!]^\ast \gets [\![\U]\!]$
\end{algorithm}
\vspace{-4mm}

\begin{remark}[Flagifying PGA or Tangent-PCA]
    True flagification of exact PGA is a difficult task. 
    While~\cref{eq: tangent_flag} resembles a flag structure, this is not the case for the nonlinear submanifolds in~\cref{eq:subman_flag}~\cite{haller2020nonlinear,haller2023weighted,ciuclea2023shape}.
    Instead, we will focus on its tangent approximations, where we map the data to the tangent space of the mean and perform (weighted) \fPCA~along with \fRPCA, \fWPCA, and \fDPCP~in the tangent space, just like in \TPCA~(\cref{rem:TPCA}).
\end{remark}

\begin{remark}[Computational Complexity (CC)]
    \cref{alg: flagified pca} for \fRPCA, \fWPCA, and \fDPCP~ has a CC of $O(N_oM)$ and~\cref{alg: flagified pca} in $\TxM$ for \fTRPCA, \fTWPCA, and \fTDPCP~has a CC of $O(N_{\mu}pn n_k^2) + O(N_oM)$, where $N_o$ is the number of iterations of the outer loop, $N_{\mu}$ is the number of iterations of the Karcher median, $M$ is the CC of Stiefel CGD, $p$ is the number of points, and the flag is of type $(n_1,n_2,\dots,n_k;n)$.
\end{remark}

\begin{remark}[Flagifying Tangent (Dual)-PCA]
All these flagified PCA formulations can be run in the tangent space of a manifold centroid, producing a corresponding \emph{tangent} version. Following the same convention, we dub these \fTPCA, \fTRPCA, \fTWPCA, and \fTDPCP~and propose an algorithm for their computation in the supplementary.
\end{remark}
\setlength{\textfloatsep}{8pt}
\begin{remark}[Even further fPCA variants]
As summarized in~\cref{tab:fl_type}, optimizing over $\flag(1,2,\dots,k;n)$ recovers $L_1$ versions of robust PCA while $\flag(k;n)$ recovers $L_2$ versions, in particular \RPCA, \WPCA, and \DPCP.
Naturally, one can employ other flag types recovering robust PCAs ``in between'' $L_1$ and $L_2$ that differ from $L_p$, $1<p<2$. Moreover, these flagified PCA formulations can be run in the tangent space of a manifold centroid to recover tangent robust $L_1$ and $L_2$ principal directions, and even ones in between. These generalizations immediately produce a plethora of novel dimensionality reduction algorithms. 
While we glimpse their potential advantages in~\cref{sec:eucPD}, we leave their thorough investigation for a future study.
\end{remark}

%% file: content/results.tex
\vspace{-3mm}\section{Results}\label{sec:results}
\vspace{-3mm}
\paragraph{Baselines}
Our algorithm results in a family of novel \PCA/\TPCA~algorithms (\cf~\cref{tab:pca_summary} in blue). We compare these to their known versions using state-of-the-art implementations. In particular, we use the bit-flipping algorithm of~\cite{markopoulos2017efficient} for \LoneRPCA, the alternating scheme of~\cite{wang2017ell} for \LtwoWPCA, and the iteratively reweighted algorithm of~\cite{vidal2018dpcp} (DPCP-IRLS) for \LtwoDPCP. Finally, we use the Pymanopt~\cite{townsend2016pymanopt} implementations for Stiefel CGD and Riemannian Trust Region (RTR) methods on flag manifolds~\cite{nguyen2023operator} to directly optimize the objectives in~\cref{prop:fPCA_all}. 

\paragraph{Implementation details}
We always initialize~\cref{alg: flagified pca} randomly and determine convergence either if we reach a maximum number of iterations (max. iters. of 50) or meet at least one of $\left| f([\![\U^{(m)}]\!]) - f([\![\U^{(m+1)}]\!]) \right| < 10^{-9}$ or $d_c([\![\U^{(m)}]\!],[\![\U^{(m+1)}]\!]) < 10^{-9}$ where $d_c(\cdot, \cdot)$ is the chordal distance on $\flag(n+1)$~\cite{pitaval2013flag}. Karcher's mean/median convergence parameter is $10^{-8}$, and step size is $0.05$. 
All algorithms are run on a $2020$ M$1$ MacBook Pro.

\paragraph{Outlier detection}
Euclidean formulations of PCA yield the residuals of $\|\x_j - \U \U^T \x_j \|_2$ (for \WPCA, \RPCA, and \PCA) and $\| \B \x_j\|_2$ (for \DPCP). To predict labels for outliers, we normalize these residuals between $[0,1]$ and decide on a threshold during AUC computation. Non-Euclidean versions, (\fTWPCA, \fTRPCA, and \TPCA), given $k$ flattened principal directions $\U = [\bu_1,\dots, \bu_k]$ at a base point $\x \in \Man$, we compute $\pi_{\U}(x_j) = \U \U^T \x_j$ and reshape it into $\bv_j \in \TxM$.
The predicted label is then obtained from the reconstruction error for $\x_j$ by thresholding the manifold distance $d(\x_j, \hat{\x}_j)$, where $\hat{\x}_j = \Exp_{\x}(\bv_j)$. The predictions for \fTDPCP~follows a slightly different scheme which uses flattened estimations for the dual principal directions in $\B = [\vb_1, \dots, \vb_k]$ and the data in the tangent space $\{\bv_j\}$. The predicted label for point $j$ is obtained by thresholding $\| \B \bv_j\|_2$.

\subsection{Evaluating Euclidean Principal Directions}
\label{sec:eucPD}
\noindent\textbf{Can flagified PCAs recover specified algorithms?}
To ensure that our robust algorithms in~\cref{alg: flagified pca} can recover traditional, specific PCA variants, we compare \foneRPCA~to \LoneRPCA, \ftwoWPCA~to \LtwoWPCA, and \ftwoDPCP~to \LtwoDPCP~in~\cref{tab:compare} (with $200$ max. iters.) by computing the first $k=2$ principal directions of $\{ \x_i\}_{i=1}^{100} \in \R^{5}$ where $\x_i\sim\mathcal{U}[0,1)$ is sampled uniformly.
As seen, our algorithms converge to similar objective values as the baselines 
while \LtwoWPCA~and \LtwoDPCP~run faster than the flag versions. Yet, the novel \foneRPCA~is much faster than \LoneRPCA~and \ftwoDPCP~converges to a more optimal objective albeit being initialized randomly, as opposed to SVD-initialization of \LtwoDPCP.

\begin{table}[t!]
    \centering
    \begin{tabular}{l@{\:}c@{\:}|@{\:}c|c@{\:}|@{\:}c|c@{\:}|@{\:}c}
        & \multicolumn{2}{c}{\LoneRPCA} & \multicolumn{2}{c}{\LtwoWPCA} & \multicolumn{2}{c}{\LtwoDPCP} \\
        \cmidrule{2-7}
        & Obj.$\uparrow$ & Time & Obj.$\downarrow$ & Time & Obj.$\downarrow$ & Time \\
        \midrule
        Baseline\,\,\, & $\mathbf{54.69}$ & $70.66$ & $\mathbf{42.89}$ & $\mathbf{0.24}$ & $34.83$ & $\mathbf{0.26}$\\
        Flag (\cref{alg: flagified pca}) & $54.66$ & $\mathbf{0.19}$ & $42.92$ & $0.45$ & $\mathbf{34.66}$ & $0.38$\\
        \bottomrule
    \end{tabular}
    \caption{Objective function values and run times comparing \foneRPCA/ \ftwoWPCA/ \ftwoDPCP~found with~\cref{alg: flagified pca} to baselines \LoneRPCA/ \LtwoWPCA/ \LtwoDPCP~respectively.\vspace{-1mm}}
    \label{tab:compare}
\end{table}

\vspace{0.3mm}

\noindent\textbf{Is our algorithm advantageous to direct optimization on manifolds?}
We compare~\cref{alg: flagified pca} to direct optimization with Stiefel CGD and Flag RTR on data $\{ \x_i\}_{i=1}^{30} \in \R^{4}$ where $\x_i \sim \mathcal{U}[0,1)$.~\cref{fig:pymanopt_compare_fl} presents run times and objective values attained when computing the first $k=2$ principal directions via \foneRPCA, \foneWPCA, and \foneDPCP~with $20$ random initializations. Our algorithms converge faster and to more optimal objective values than naive Stiefel-CGD and Flag-RTR.

\begin{figure}
    \centering\vspace{-3mm}
    \includegraphics[trim=0mm 2.8mm 0mm 1mm,clip,width = \linewidth]{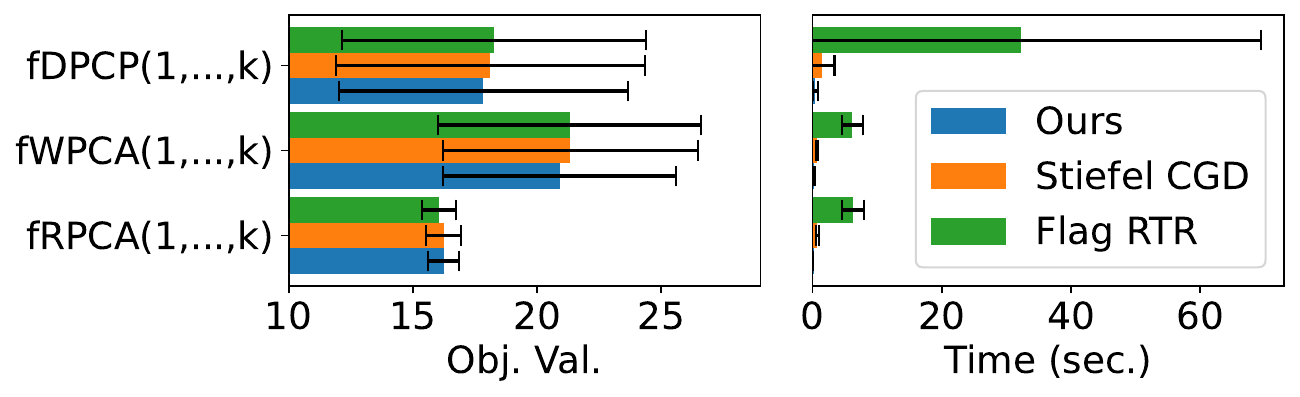}
    \caption{\cref{alg: flagified pca} converges faster to more optimal cost values compared to Stiefel CGD or Flag RTR.\vspace{-3mm}}
    \label{fig:pymanopt_compare_fl}
\end{figure}

\paragraph{Outlier detection on remote sensing data}
We use the UCMercedLandUseDataset~\cite{yang2010bag} with $100$ inlier `runway' and introduce outlier `mobilehomepark' images. We use the benchmark RPCA by Candès~\etal~\cite{candes2011robust}. Results in~\cref{fig:images} {\textcolor{red}{(top)}} indicate a slight yet consistent increase in performance using novel robust \fDPCP$(1,40)$.

\paragraph{Outlier detection on Cropped YaleFaceDB-B}
Similar to DPCP~\cite{Tsakiris_2015_ICCV_Workshops}, we use the $64$ illuminations of one face from YaleFaceDB-B~\cite{cropped_yaleb} as inliers and introduce outliers as random images from Caltech101~\cite{caltech101}. Results in~\cref{fig:images} {\textcolor{red}{(bottom)}} indicate that our robust flag methods are advantageous and the dual variant dominates as the outlier contamination increases.

\begin{figure}[t!]
    \begin{subfigure}[b]{\columnwidth}
        \includegraphics[width = \linewidth]{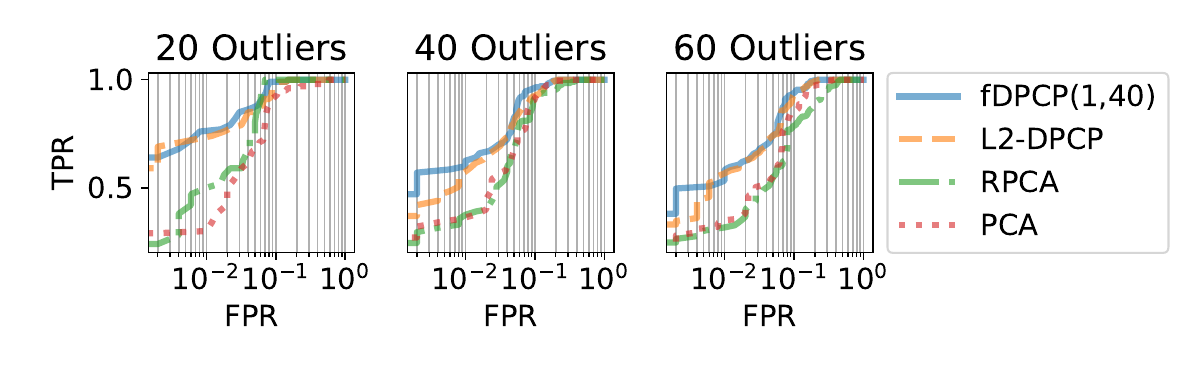}
    \end{subfigure}
    \hfill
    \begin{subfigure}[b]{\columnwidth}
    \vspace{-2mm}
        \includegraphics[width = \linewidth]{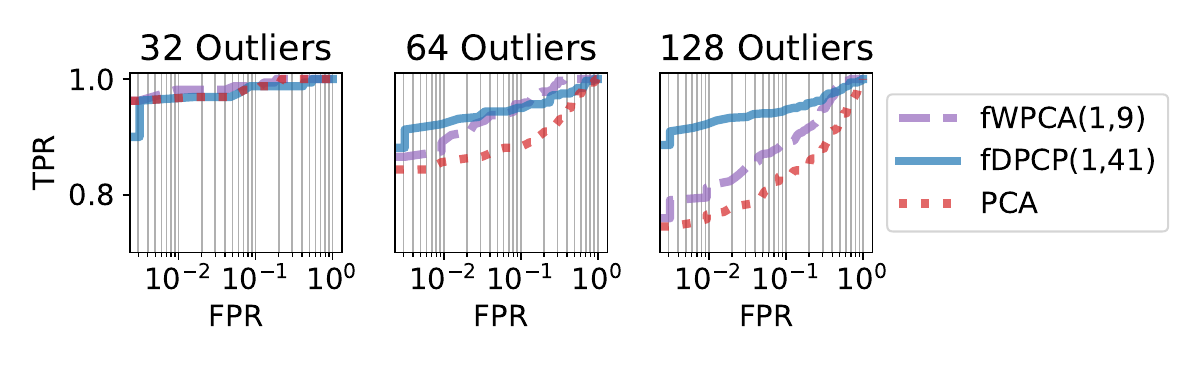}\vspace{-3mm}
    \end{subfigure}
    \caption{Average ROC curves over five trials of outlier samples for UCMercedLandUse (\textbf{top}) and YaleFaceDB-B (\textbf{bottom}). All data is reshaped and projected to $\R^{50}$ before outlier detection.}
    \label{fig:images}
\end{figure}

\vspace{-2mm}
\subsection{Evaluating Non-Euclidean Extensions}\label{sec: exp-noneuc}
\vspace{-2mm}
We now evaluate flagified tangent-PCA and its robust variants starting with a synthetic evaluation of the sphere and Grassmannian before moving to real datasets. See the supplementary for additional experiments.

\paragraph{Convergence on $4$-sphere}
To sample a dataset of inliers and outliers on the $4$-sphere $\mathbb{S}^4 = \left \{ \x \in \R^5 : \| \x \|_2 = 1 \right\}$ (see supplementary for details). Then we compute the first $k=2$ principal directions of \foneTRPCA, \foneTWPCA, and \foneTDPCP~and plot objective values as Euclidean optimizations in the tangent space of the Karcher median at each iteration of~\cref{alg: flagified pca} in~\cref{fig: all_convergence}. All methods converge quickly, while the spread of objective function values due to initializations decreases.

\begin{figure}[t]
    \centering
    \includegraphics[trim=0mm 4mm 0mm 1mm,clip, width =\linewidth]{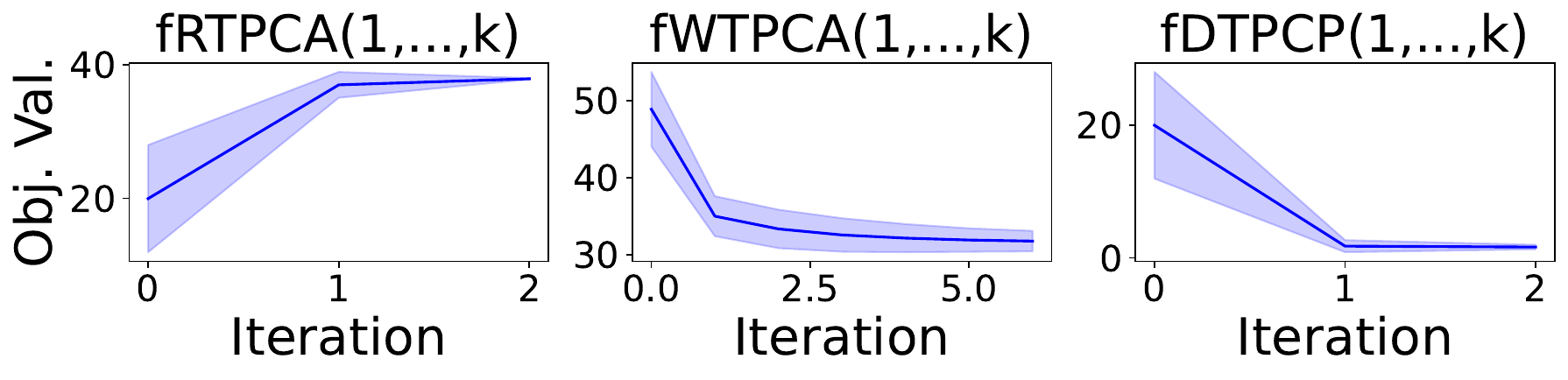}
    \caption{$50$ random initializations of \fTPCA~ variations. The blue line is the mean and the shaded region is the standard deviation. The $x$-axis of this plot is the number of iterations of~\cref{alg: flagified pca} performed in the tangent space.\vspace{-3mm}}
    \label{fig: all_convergence}
\end{figure}

\paragraph{Outlier detection on $Gr(2,4)$}
To compare between different flag type realizations of flagified robust PCAs, we now synthesize data with inliers and outliers on $Gr(2,4)$, the set of all $2$-planes in $\R^4$ represented as $Gr(2,4) = \left \{ [\X]: \X \in \R^{4 \times 2} \text{ and }\X^T \X = \I  \right \}$~\cite{edelman1998geometry}. To do so, consider two random points $[\X],[\Y] \in Gr(2,4)$ acting as centroids for inliers and outliers, respectively.  
Inliers are sampled as $\Exp_{[\X]}(a\V_i)$ where $a \sim \mathcal{U}[0,1)$, $\V_1, \V_2 \in \mathcal{T}_{[\X]}(Gr(2,4))$ are two random tangent vectors and $\Exp_{[\X]}$ is the exp-map of $Gr(2,4)$. We randomly choose $i\in \{1,2\}$. 
Outliers are sampled similarly as $\Exp_{[\Y]}(b\V)$ where $b \sim \mathcal{U}[0,0.1)$ and gradually added to the dataset.
~\cref{fig:out_auc_gr24_man} plots the AUCs for outlier detection using the first $k=2$ principal directions. \foneTDPCP~produces the highest AUC and is more stable to the presence of outliers. We also found that Euclidean \PCA~variants with the same data produce lower AUC (see supplementary).

\begin{figure}[t]
    \centering
    \includegraphics[width =\linewidth, trim=0mm 1.6mm 0mm 4mm, clip]{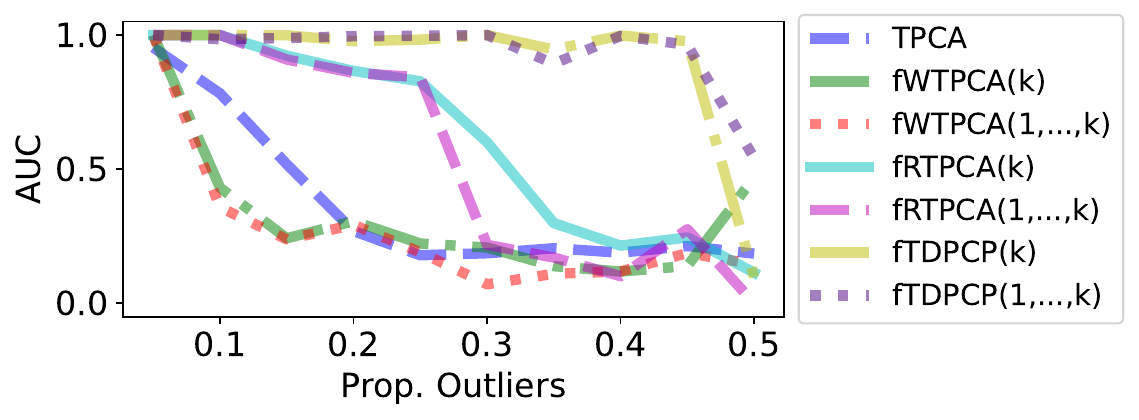}
    \caption{AUC of different algorithms for outlier detection using the first $k=2$ principal directions of outlier-contaminated data on $Gr(2,4)$. All iterative variants are optimized with $100$ max. iters.}
    \label{fig:out_auc_gr24_man}
\end{figure}

\paragraph{Outlier detection and reconstruction on Kendall pre-shape space}
We use an outlier-contaminated version of the 2D Hands~\cite{stegmann2002} to probe the performance on a real dataset. We represent the $44$ total inlier Procrustes-aligned hands
and added outliers in the Kendall pre-shape space~\cite{kendall1984shape}: $\Sigma_2^{56} := \left \{ \X \in \R^{56 \times 2} \: : \: \| \X\|_F = 1 \text{ and } \sum_{i=1}^{56} \x_i = 0 \right \}$.
We sample outliers as open ellipses with axes sampled from $\mathcal{N}(.4,.5)$, centers from $\mathcal{N}(0,.1)$, and a hole that is $\approx 6.8\%$ of the entire ellipse. We project these outliers onto $\Sigma_2^{56}$ by normalization and mean-centering. 
~\cref{fig: hand_pred_ablation} reports the AUC on outlier detection as we gradually add outliers. \fTDPCP~has the best outlier detections for both flag variants followed by \fTRPCA, \fTWPCA, and \TPCA~with flag type $(1,2,\dots,k)$ producing different AUC than $\flag(k)$ (\cf~\cref{tab:pca_summary}). All algorithms in these experiments are initialized with the SVD.
\begin{figure}[t]
    \centering
    \includegraphics[width = \linewidth]{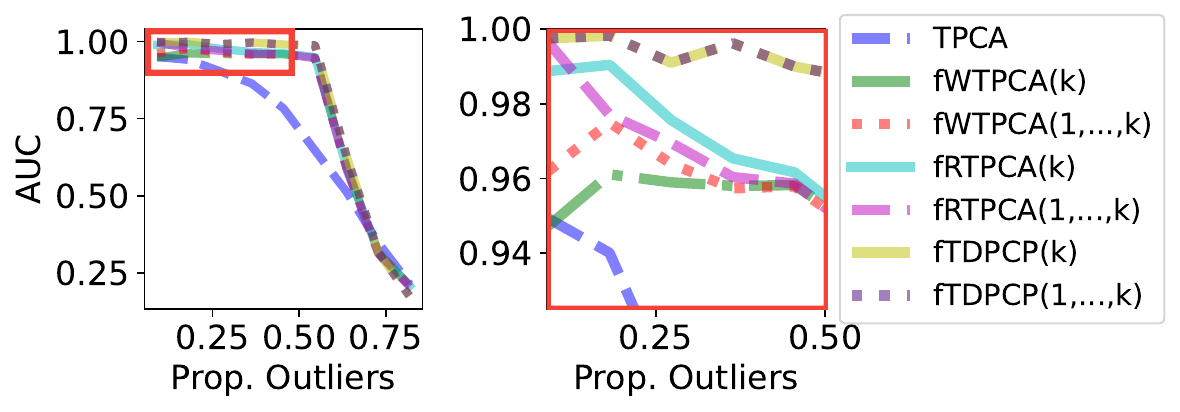}
    \caption{Mean AUC for outlier predictions using the first $k=4$ principal directions where we gradually add outlier ellipses to the $2$D Hands dataset. The mean is over $20$ trials of adding outliers.\vspace{-3mm}}
    \label{fig: hand_pred_ablation}
\end{figure}
We further consider a dataset with $30$ outliers to isolate the hands (inliers). We run \TPCA~with $k=4$ principal directions to reconstruct the first hand in~\cref{fig:rec_with_out_removal}. In a slight abuse of notation, we reconstruct a hand $\x \in \Sigma_2^{56}$ using $k=4$ principal tangent directions $\{\bu_1, \dots, \bu_4\} \in \mathcal{T}_{\bm{\mu}}\left(\Sigma_2^{56}\right)$ as $\hat{\x} = \Log_{\mu} \left( \U \U^T \Exp_{\mu}(\x) \right)$ where $\U = [\bu_1, \dots, \bu_4]$. Since \foneTDPCP~and \ftwoTDPCP~almost perfectly detect all the outliers, they produce the best reconstructions.
\begin{figure}
    \centering
    \includegraphics[width = \linewidth]{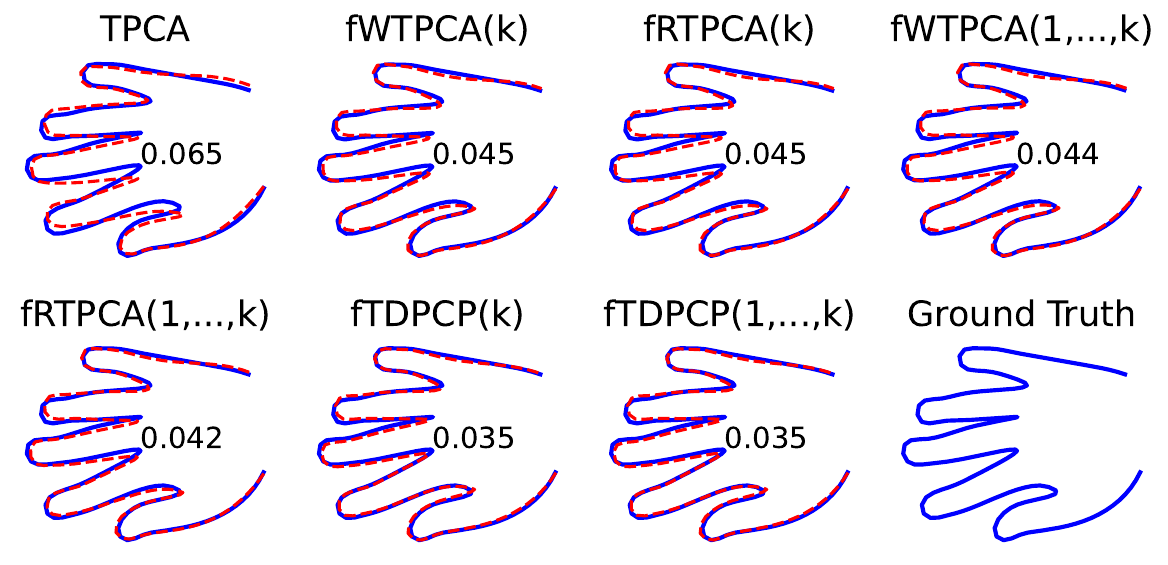}
    \caption{Reconstructions using a \PCA~of the inliers detected by variants of \fTWPCA, \fTRPCA, \fTDPCP~and \TPCA~on a the $2$D Hands dataset with $28$ hands and $16$ ellipses. The reconstruction error is reported inside each hand.}\label{fig:rec_with_out_removal}
\end{figure}

%% file: content/conclusion.tex
\section{Conclusion}
Having fun with flags, we have presented a unifying flag-manifold-based framework for computing robust principal directions of Euclidean and non-Euclidean data. Covering PCA, Dual-PCA, and their tangent versions in the same framework has given us a generalization power to develop novel, manifold-aware outlier detection and dimensionality reduction algorithms, either by modifying flag-type or by altering norms. We further devised practical algorithms on Stiefel manifolds to efficiently compute these robust directions without requiring direct optimization on the flag manifold. 
Our experimental evaluations revealed that new variants of robust and dual \PCA/tangent \PCA~discovered in our framework can be useful in a variety of applications.

\paragraph{Limitations \& future Work}
We cannot handle \emph{non-linear flags}~\cite{haller2023weighted} and hence cannot cover nested spheres/hyperbolic spaces~\cite{dryden2019principal,jung2012analysis,fan2022nested}.
We have also not included Barycentric subspaces~\etal~\cite{pennec2018barycentric}. We leave these for a future study.

%% file: content/appendix.tex
\section{Theoretical Justifications \& Discussions}\label{sec:proofs}

\paragraph{On the unifying aspects of our framework}
In our framework, the link between RPCA \& Dual-PCA, established also in the discussed earlier works, emerges as a by-product of our unifying formulation. To elucidate, our flag-based framework allows for: (i) extending \DPCP~to manifold-valued data (\fTDPCP), (ii) interpolating between $L_1$/$L_2$--DPCP~via the use of non-trivial flag types, and (iii) an efficient algorithms for computing flag--(tangent) DPCP for any flag type. To the best of our knowledge, Alg. 1 (main paper) is the only method for finding non-trivial flags of robust directions and when used for both \fRPCA~\& \fWPCA. 



\subsection{Proof of Prop. 3}
Let us recall the proposition before delving into the proof.
\begin{prop}[Stiefel optimization of (weighted) fPCA]
    Suppose we have weights $\{w_{ij}\}_{i=1,j=1}^{i=k,j=p}$ for a dataset $\{ \x_j \}_{j=1}^p \subset \R^n$ along with a flag type $(n_1,n_2,\dots, n_k; n)$. We store the weights in the diagonal weight matrices $\{\W_i\}_{i=1}^k$ with diagonals $(\W_i)_{jj} = w_{ij}$. If
    \begin{equation}
        \U^\ast  = \argmax_{\U \in St(n_k,n)} \sum_{i=1}^k \tr \left( \U^T \X \W_i \X^T \U \I_i \right)
    \end{equation}
    where $\I_i$ is determined as a function of the flag signature. For example, for $\flag(n+1)$: 
    \begin{equation}
    (\I_i)_{l,s} = 
    \begin{cases}
        1, & l = s \in \{ n_{i-1} + 1, 
 n_{i-1} + 2, \dots, n_i\} \\
        0, &\mathrm{ otherwise}\nonumber\\
    \end{cases}    
\end{equation}
    Then $[\![\U^\ast ]\!] = [\![\U ]\!]^\ast$ is the weighted fPCA of the data with the given weights (e.g., solves~\cref{eq:wflagpca}) as long as we restrict ourselves to a region on $\flag(n+1)$ and $St(n_k,n)$ where weighted fPCA is convex.
\end{prop}
\begin{proof}
    First we will show that the flag and Stiefel objective functions are equivalent. Take 
    \begin{equation}
        [\![ \U ]\!] \in \flag(n+1) = \flag(n_1,n_2,\dots,n_k;n). 
    \end{equation}
    We decompose $\U = [\U_1,\U_2,\dots,\U_k]$ where $\U_i \in \R^{n \times m_i}$ and $\sum_{l=1}^i m_l = n_i$. Using $\I_i$ (defined above) we have $\U \I_i \U^T = \U_i$. 
    
    Recall the objective function for both \fRPCA~and \fDPCP~is
    \begin{align}
        \mathbb{E}_j \left[ \sum_{i=1}^k w_{ij}\| \pi_{\U_i} (\x_j)\|_2^2 \right] 
        &= \sum_{j=1}^p \sum_{i=1}^k w_{ij}\| \pi_{\U_i} (\x_j)\|_2^2, \label{eq: var_opt_supp}\\
        &= \sum_{j=1}^p \sum_{i=1}^k w_{ij}\| \U_i \U_i^T \x_j \|_2^2
    \end{align}
    Using the definition of norms and $\U_i^T \U_i = \I$,~\cref{eq: var_opt_supp} is equivalent to 
    \begin{equation}
        \sum_{j=1}^p \sum_{i=1}^k w_{ij}\tr \left( \x_j^T \U_i \U_i^T  \x_j \right)
    \end{equation}
    Now, using properties of trace, matrix multiplication, and our handy $\{\I_i\}_{i=1}^k$ we reach our desired result
    \begin{align}
        &\sum_{j=1}^p \sum_{i=1}^k w_{ij}\tr \left( \U_i^T \x_j  \x_j^T \U_i\right), \\
        &=  \sum_{i=1}^k \tr \left( \U_i^T \left( \sum_{j=1}^p w_{ij} \x_j  \x_j^T \right) \U_i\right), \\
        &=  \sum_{i=1}^k \tr \left( \U_i^T \left(\X \W_i \X^T\right) \U_i\right), \\
        &=  \sum_{i=1}^k \tr \left( \U_i \U_i^T \X \W_i \X^T \right), \\
        &=  \sum_{i=1}^k \tr \left( \U \I_i \U^T \X \W_i \X^T \right), \\
        &=  \sum_{i=1}^k \tr \left(  \U^T \X \W_i \X^T\U \I_i  \right).\label{eq: var_opt_supp1} 
    \end{align} 

So we have shown that the flag and Stiefel objective functions are equivalent.

Finally, we show $[\![ \U^\ast ]\!]= [\![ \U]\!]^\ast$. Notice that the objective function for weighted flag PCA is invariant to different flag manifold representatives. First, let $f$ denote the objective function in~\cref{eq: var_opt_supp1}. Suppose $\U^\ast$ solves $\argmax_{\Y \in St(n_k, n)} f(\Y)$. Then take some other representative for $[\![\U^\ast]\!]$, namely $\U^\ast \M$ where
\begin{equation}
    \M = \begin{bmatrix}
    \M_1 & \bm{0} & \bm{0} & \bm{0}\\
    \bm{0} & \M_2 & \bm{0} &\bm{0}\\
    \vdots & \vdots & \ddots & \vdots\\
    \bm{0} & \bm{0} & \dots & \M_k  \\
         \end{bmatrix}
         \text{ and } \M_1 \in O(m_i).
\end{equation}
Then $f(\U^\ast \M) = f(\U^\ast)$ because
\begin{align}
    f(\U^\ast \M) &= \sum_{i=1}^k \tr ((\U_i^\ast\M)^T \X \W_i \X^T (\U_i^\ast\M)),\\
     &= \sum_{i=1}^k \tr (\U_i^\ast \M\M^T{\U_i^\ast}^T \X \W_i \X^T),\\
     &= \sum_{i=1}^k \tr (\U_i^\ast{\U_i^\ast}^T \X \W_i \X^T),\\
     &= \sum_{i=1}^k \tr ({\U_i^\ast}^T \X \W_i \X^T\U_i^\ast),\\
     &= f(\U^\ast).
\end{align}
So $f(\cdot)$ has the same value for any representative for $[\![ \U^\ast ]\!]$. Since $f(\U^\ast) \geq f(\Y)$ for all $\Y \in St(n_k,n)$, then 
\begin{equation}
    f(\U^\ast\M)= f(\U^\ast) \geq f(\Y) = f(\U^\ast\bm{O})
\end{equation}
for all $[\![\Y ]\!] \in \flag(n+1)$ where $\bm{O}$ is of the same block structure as $\M$. 

Recall $[\![\U]\!]^\ast \in \flag(n+1)$ maximizes $f$, so $f(\U) \geq f(\Y)$ for all $[\! [\Y ] \! ] \in \flag(n+1)$ and since $f(\cdot)$ has the same value for any representative of $[\! [\Y ] \! ]$, we have $f(\U) \geq f(\Y)$ for all $\Y \in St(n_k,n)$. 

Recall, that $f(\U^\ast) \geq f(\Y)$ for all $\Y \in St(n_k,n)$. So $f(\U^\ast) = f(\U)$. Since $f$ has a unique maximizer over $\flag(n+1)$, we have $[\![\U^\ast]\!] = [\![ \U]\!]^\ast = \argmax_{[\![\Y ]\!] \in \flag(n+1)} f(\Y)$.
\end{proof}

\subsection{Proof of Prop. 4}
Let us recall the proposition before delving into the proof.
\begin{prop}[Stiefel optimization for flagified Robust (Dual-)PCAs]
   We can formulate \fRPCA, \fWPCA, \fDPCP, and \fWDPCP~as optimization problems over the Stiefel manifold using $[\![ \U ] \!]^\ast = [\![\U^\ast] \!]$ and the following: 
    \begin{align}
    &\U^\star = \\
    &\begin{cases}
            \argmax\limits_{\U \in St(n, n_k)}  \sum_{i=1}^k \tr \left(  \U^T \Pm_i^+\U \I_i  \right), & \mathrm{(fRPCA)}\\
            \argmin\limits_{\U \in St(n, n_k)}\sum_{i=1}^k \tr \left(\Pm_i^- -  \U^T \Pm_i^-\U \I_i  \right), & \mathrm{(fWPCA)}
            \end{cases}\nonumber\\
    &\nonumber\\
    &\U^\star = \\
    &\begin{cases}
            \argmin\limits_{\U \in St(n, n_k)} \sum_{i=1}^k \tr \left(  \U^T \Pm_i^+ \U \I_i  \right), & \mathrm{(fDPCP)}\\
            \argmax\limits_{\U \in St(n, n_k)} \sum_{i=1}^k\tr \left(\Pm_i^- -  \U^T \Pm_i^-\U \I_i  \right) & \mathrm{(fWDPCP)}
            \end{cases}\nonumber
    \end{align}
    where $\Pm^-=\X \W^-_i([\![\U]\!]) \X^T$, $\Pm^+=\X \W^+_i([\![\U]\!]) \X^T$ and $\W^-_i([\![\U]\!])$, $\W^+_i([\![\U]\!])$ are defined in~\cref{tab:weight_org} as long as we restrict ourselves to a region on $\flag(n+1)$ and $St(n_k,n)$ where flag robust and dual PCAs are convex.
\end{prop}

\begin{proof}
    First, we write the objective functions for \fRPCA~and \fDPCP~over $St(n_k,n)$  using~\cref{eq: max_weights} to define each $\W_i^+$ as
    \begin{align}
        f^+(\U) &=\mathbb{E} \left[ \sum_{i=1}^k\| \pi_{\U_i} (\x_j)\|_2 \right],\\
        &= \sum_{j=1}^p \sum_{i=1}^k\| \pi_{\U_i} (\x_j)\|_2,\\
        &= \sum_{j=1}^p \sum_{i=1}^k\sqrt{\tr \left( \x_j^T \U_i \U_i^T  \x_j \right)},\\
        &=\sum_{j=1}^p \sum_{i=1}^k \sqrt{\tr \left( \U_i^T \x_j  \x_j^T \U_i\right)}\\
        &=\sum_{j=1}^p \sum_{i=1}^k \sqrt{\tr \left( \U^T \x_j  \x_j^T \U \I_i \right)},\\
        &=\sum_{j=1}^p \sum_{i=1}^k \frac{\tr \left( \U^T \x_j  \x_j^T \U \I_i \right)}{\sqrt{\tr \left( \U^T \x_j  \x_j^T \U \I_i \right)}},\\
        &= \sum_{i=1}^k \tr \left( \U^T \sum_{j=1}^p \frac{\x_j  \x_j^T}{{\| \U \I_i\U^T \x_j\|_2}} \U \I_i \right),\\
        &=\sum_{i=1}^k \tr \left( \U^T \X \W_i^+ \X^T \U \I_i \right),\\
        &=\sum_{i=1}^k \tr \left( \U^T \Pm_i^+ \U \I_i \right).
    \end{align}
    
    Now we write the objective functions for \fWPCA~and \fWDPCP~over $St(n_k,n)$ using~\cref{eq: min_weights} to define each $\W_i^-$ as 
    \begin{align}
        f^-(\U) &= \mathbb{E} \left[ \sum_{i=1}^k\| \x_j - \pi_{\U_i} (\x_j)\|_2 \right], \\
        &= \sum_{j=1}^p \sum_{i=1}^k\| \x_j - \pi_{\U_i} (\x_j)\|_2,\\
        &= \sum_{j=1}^p \sum_{i=1}^k\sqrt{\tr \left( \x_j^T\x_j - \x_j^T\U_i \U_i^T  \x_j \right)},\\
        &= \sum_{j=1}^p \sum_{i=1}^k \sqrt{\x_j^T\x_j - \tr \left(\U_i^T \x_j  \x_j^T \U_i\right)},\\
        &= \sum_{j=1}^p \sum_{i=1}^k \sqrt{\x_j^T\x_j - \tr \left(\U^T \x_j  \x_j^T \U \I_i\right)},
    \end{align}
    \begin{align}
        &= \sum_{j=1}^p \sum_{i=1}^k \frac{\x_j^T\x_j - \tr \left(\U^T \x_j  \x_j^T \U \I_i\right)}{\sqrt{\x_j^T\x_j - \tr \left(\U^T \x_j  \x_j^T \U \I_i\right)}},\\
        &=   \sum_{j=1}^p \frac{\x_j  \x_j^T}{\| \x_j - \U \I_i\U^T \x_j\|_2}\\
        &-\sum_{i=1}^k \tr \left( \U^T \sum_{j=1}^p \frac{\x_j  \x_j^T}{{\| \x_j - \U \I_i\U^T \x_j\|_2}} \U \I_i \right)\\
        &=\sum_{i=1}^k \tr\left(\X \W_i^- \X^T - \U^T \X \W_i \X^T \U \I_i \right),\\
         &=\sum_{i=1}^k \tr\left(\Pm^- -  \U^T \Pm^- \U \I_i \right).
    \end{align}

    Now, we can write the Lagrangians for these problems with the symmetric matrix of Lagrange multipliers $\bm{\Lambda}$ as
    \begin{align*}
        \mathcal{L}^+ (\U) &= f^+(\U) + \tr(\bm{\Lambda}^+(\I - \U^T \U )),\\
        \mathcal{L}^- (\U) &= f^-(\U)+ \tr(\bm{\Lambda}^-(\I - \U^T \U )).
    \end{align*}
    Then, we collect our gradients in the following equations
    \begin{align}
        \nabla_{\U} \mathcal{L}^+
        &=  \sum_{j=1}^p \sum_{i=1}^k \frac{\x_j \x_j^T \U \I_i}{\| \U \I_i \U^T \x_j\|_2} -2 \U \bm{\Lambda}^+\\
         \nabla_{\U} \mathcal{L}^-
        &= - \sum_{j=1}^p \sum_{i=1}^k \frac{\x_j \x_j^T \U \I_i}{\| \x_j \x_j^T - \U \I_i\U^T \x_j\|_2} - 2 \U \bm{\Lambda}^- \\
         \nabla_{\bm{\Lambda}^+} \mathcal{L}^+ &=  \nabla_{\bm{\Lambda}^-} \mathcal{L}^- =  \I - \U^T \U.\\
    \end{align}

    Then setting $\nabla_{\U} \mathcal{L}_1 = \bm{0}$, $\nabla_{\bm{\Lambda}_1} \mathcal{L}_1= \bm{0}$, left multiplying by $\U^T$, and playing with properties of trace results in 
    \begin{align}
        \sum_{i=1}^k \tr \left( \U^T \X \W_i \X^T \U \I_i \right) &= 2\tr( \bm{\Lambda}^+), \label{eq: lagr_max}\\
        \sum_{i=1}^k \tr \left( \U^T \X \W_i \X^T \U \I_i \right) &= -2\tr( \bm{\Lambda}^-).\label{eq: lagr_min}
    \end{align}
    Then we have the following cases: we choose
    \begin{itemize}
        \item (\fRPCA) $\U^\ast$ to maximize $\tr( \bm{\Lambda}^+)$ so that we maximize $f^+$,
        \item (\fDPCP) $\U^\ast$ to minimize $\tr( \bm{\Lambda}^+)$ so that we minimize $f^+$,
        \item (\fWPCA) $\U^\ast$ to minimize $-\tr( \bm{\Lambda}^-)$ so that we minimize $f^-$,
        \item (\fWDPCP) $\U^\ast$ to maximize $-\tr( \bm{\Lambda}^-)$ so that we maximize $f^-$.
    \end{itemize}
 
\begin{align}
        \sum_{j=1}^p \sum_{i=1}^k \frac{\x_j \x_j^T \U \I_i}{\| \U_i^T \x_j\|_2} &=  \bm{\Lambda}_1 \U, \\
        \sum_{j=1}^p \sum_{i=1}^k \frac{\U^T \x_j \x_j^T \U \I_i}{\| \U_i^T \x_j\|_2} &=  \bm{\Lambda}_1, \\
        \sum_{j=1}^p \sum_{i=1}^k (\W_i)_{jj} \U^T \x_j \x_j^T \U \I_i &=  \bm{\Lambda}_1, \\
        \tr\left(\sum_{i=1}^k \U^T \left( \sum_{j=1}^p (\W_i)_{jj} \x_j \x_j^T \right)  \U \I_i\right) &=  \tr(\bm{\Lambda}_1), \\
        \sum_{i=1}^k \tr\left(\U^T \left( \sum_{j=1}^p (\W_i)_{jj} \x_j \x_j^T \right)  \U \I_i\right) &=  \tr(\bm{\Lambda}_1), \\
        h_{[\![ U]\!]}(\U) &=  \tr(\bm{\Lambda}_1).
\end{align}
        Similarly, setting $\nabla_{\U} \mathcal{L}_2 = \bm{0}$, $\nabla_{\bm{\Lambda}_2} \mathcal{L}_2= \bm{0}$ and leveraging~\cref{eq: min_weights} to define $\{\W_i\}_i$ results in
\begin{align}
        \sum_{j=1}^p \sum_{i=1}^k \frac{\x_j \x_j^T \U \I_i}{\| \x_j  - \U_i \U_i^T \x_j\|_2} &=  \bm{\Lambda}_2 \U, \\
        -\sum_{j=1}^p \sum_{i=1}^k \frac{\U^T \x_j \x_j^T \U \I_i}{\| \x_j  - \U_i \U_i^T \x_j\|_2} &=  \bm{\Lambda}_2, \\
        -\sum_{j=1}^p \sum_{i=1}^k (\W_i)_{jj} \U^T \x_j \x_j^T \U \I_i &=  \bm{\Lambda}_2, \\
        -\sum_{i=1}^k \tr(\U^T \left( \X \W_i \X^T \right)  \U \I_i) &=  \tr(\bm{\Lambda}_2),\\
        -h_{[\![ U]\!]}(\U) &=  \tr(\bm{\Lambda}_2).
\end{align}

    Finally, using a similar argument to that for the proof of the Stiefel optimization of \fPCA~leveraging assumed convexity, we have that $[ \! [ \U^\ast ] \! ] = [ \! [ \U ] \! ]^\ast$. 
\end{proof}

\subsection{Proof of Prop. 5}
We now prove the convergence of our algorithm. Let us recall the proposition from the main paper before delving into the proof.
\begin{prop}[Convergence of \cref{alg: flagified pca} for \fDPCP]
    \cref{alg: flagified pca} for \fDPCP~converges as long as $\| \U \I_i \U^T \x_j\|_2 \geq \epsilon$ $\forall i,j$ and we restrict ourselves to a region on $\flag(n+1)$ and $St(n_k,n)$ where \fDPCP is convex.
\end{prop}
\begin{proof}
This proof follows closely to what was done in~\cite{Mankovich_2023_ICCV}. First let $f^+: \flag(n+1) \times \flag(n+1) \rightarrow \R$ denote the \fDPCP~objective function and $T: \flag(n+1) \rightarrow \flag(n+1)$ denote an iteration of~\cref{alg: flagified pca}. Then, assuming that $\| \U \I_i \U^T \x_j \|_2 \geq \epsilon$ for $i=1,2,\dots, k$ and $j = 1,2, \dots, p$, we define the function $h: \flag(d+1) \times \flag(d+1) \rightarrow \R$ as

\begin{align}
h([\![\Z]\!], [\![\U]\!]) &= \sum_{i=1}^p \tr(\Z^T \X \W_i^+([\![\U]\!]) \X^T \Z \I_i) ,\label{eq:h_def}
\end{align}
using the definition in~\cref{eq: max_weights} for $\W_i^+([\![\U]\!])$.

Some algebra reduces $h([\![\Z]\!], [\![\U]\!])$ to
\begin{equation}
    h([\![\Z]\!], [\![\U]\!]) = \sum_{i=1}^p \sum_{j=1}^k \frac{\| \Z \I_i \Z^T \x_j \|_2^2}{\| \U \I_i \U^T \x_j \|_2}.
\end{equation}

From~\cref{eq:h_def}, we see that $h([\![\Z]\!], [\![\U]\!])$ is the weighted flag PCA objective function of $\{ \x_j \}_{j=1}^p$ with weights on the diagonals of $\W_i([\![\U]\!])$. The weighted flagified orthogonal PCA (\fOPCA) optimization problem with weights in the diagonals $\W^+_i([\![\U]\!])$ can be solved using a similar algorithm to~\cref{alg: weighted flagified pca} by just minimizing instead of maximizing (see~\cref{alg: weighted orth flagified pca}). Thus minimizing $h([\![\Z]\!], [\![\U]\!])$ over $[\![\Z]\!]$ is an iteration of~\cref{alg: weighted orth flagified pca} for \fDPCP~which means

\begin{equation}
    T([\![\U]\!]) = \argmin_{[\![\Z]\!] \in \flag(d+1)} h([\![\Z]\!], [\![\U]\!]).
\end{equation}

Using this, we have
\begin{equation}
    h(T([\![\U]\!]), [\![\U]\!]) \leq h([\![\U]\!], [\![\U]\!]).
\end{equation}

By the definition of $h$
\begin{align}
    h([\![\U]\!], [\![\U]\!]) &= \sum_{i=1}^p \sum_{j=1}^k \frac{\| \U \I_i \U^T \x_j \|_2^2}{\| \U \I_i \U^T \x_j \|_2},\\
    &= \sum_{i=1}^p \sum_{j=1}^k \| \U \I_i \U^T \x_j \|_2,\\
    &= f([\![\U]\!]).
\end{align}
This means, we have
\begin{equation}\label{eq: h flagirls less}
    h(T([\![\U]\!]),  [\![\U]\!]) \leq f([\![\U]\!]).
\end{equation}

Now we use the identity from algebra: $\frac{a^2}{b} \geq 2a-b$ for any $a,b \in \R$ and $b > 0$. Let 
\begin{equation}
    a = \| \Z \I_i \Z^T \x_j \|_2 \text{ and } b =  \| \U \I_i \U^T \x_j \|_2 .
\end{equation} 
Then
\begin{align}
h([\![\Z]\!], [\![\U]\!]) &\geq 2\sum_{j=1}^p \sum_{i=1}^k\| \Z \I_i \Z^T \x_j \|_2\\
&- \sum_{j=1}^p  \sum_{i=1}^k\| \U \I_i \U^T \x_j \|_2,  \\
&= 2f([\![\Z]\!]) - f([\![\U]\!]).
\end{align}
Now, take $[\![\Z]\!] = T([\![\U]\!])$. This gives us
\begin{equation}\label{eg: h flagirls greater}
    h(T([\![\U]\!]), [\![\U]\!]) \geq 2f(T([\![\U]\!]))-  f([\![\U]\!]).
\end{equation}

Then, combining Eq.~\ref{eg: h flagirls greater} with Eq.~\ref{eq: h flagirls less}, we have

\begin{align}
    2f(T([\![\U]\!]))- f([\![\U]\!]) &\leq f([\![\U]\!]), \\
    f(T([\![\U]\!])) &\leq f([\![\U]\!]) . 
\end{align}

Finally, notice that the real sequence with terms $f^+(T([\![\U^{(m-1)}]\!])) = f^+([\![\U^{(m)}]\!]) \in \R$ is bounded below by $0$ and is decreasing. So it converges as $m  \rightarrow \infty$.

\end{proof}

\section{Further Notes on Flagified PCA}
We now generalize PCA and its variants using flags by grouping eigenvectors using the flag type. The PCA optimization problem is naturally an optimization problem on the Stiefel manifold, $St(k,n) := \{ \U \in \R^{k \times n} \: : \: \U^T \U = \I\}$. Suppose $\U = [\bu_1, \bu_2, \dots, \bu_k] \in St(k,n)$ are the $k<n$ principal components of a data matrix $\X$. These are naturally ordered according to their decreasing associated objective function values\footnote{The objective function values are also referred to as explained variances, eigenvalues and squared singular values}. This results in the nested subspace structure
\begin{equation}
    [\![ \U ]\!] = [\bu_1] \subset [\bu_1, \bu_2] \subset \cdots \subset [\bu_1, \bu_2, \dots, \bu_k] \subset \R^n.
\end{equation}
So one can think of $[\![\U ]\!] \in \flag(1,2,\dots,k; n)$, and consequently, reformulate PCA as an optimization problem over $\flag(1,2,\dots,k; n)$.  
Thinking of $\U$ as $[\![ \U ]\!]$ emphasizes the nested subspace structure of the principal components according to their associated objective function values.

What if we have multiple principal components with the same objective function value? In other words, suppose we have at least one eigenvalue of $\X \X^T$ with a geometric multiplicity greater than $1$? For example, assume our dataset has a large variance on some $2$-plane, and all other directions orthogonal to that plane have smaller, unequal variance. Then, the first two principal components, $\bu_1$ and $\bu_2$, will have the same objective function value in~\cref{eq: maxpca}. Additionally, any rotation of the two vectors within the plane $\text{span}(\bu_1, \bu_2)$ will still produce the same objective function values. So,~\cref{eq: maxpca} is no longer a convex optimization problem over $St(k,n)$ because the first two principal components are not unique. However, if we remove $[\bu_1] \subset [\bu_1, \bu_2]$ from the nested subspace structure and consider $[\![ \U ]\!] \in \flag(2,3,\dots, k;n)$ as
\begin{equation}
    [\![ \U ]\!] = [\bu_1, \bu_2] \subset [\bu_1, \bu_2, \bu_3] \subset \cdots \subset [\bu_1, \dots, \bu_k]\subset \R^n.
\end{equation}
Then we have a unique solution to~\cref{eq: maxpca} over $\flag(2,3,\dots, k;n)$ in place of $St(k,n)$. In practice, it is unlikely that we will have two eigenvectors with the same eigenvalue. However, we can consider two eigenvalues the same as long as $|\lambda_i - \lambda_j| < \epsilon$ for some $\epsilon > 0$.

Motivated by this example, we state a generalization of PCA, which optimizes over flags of a given type.
\begin{dfn}[Flagified PCA (fPCA)~\cite{pennec2018barycentric}]
A flag of principal components is the solution to:
    \begin{equation}\label{eq:flagpca}
         \argmax_{[\![\U]\!] \in \flag(n+1)}  \mathbb{E} \left[ \sum_{i=1}^k\| \pi_{\U_i} (\x_j)\|_2^2 \right]
    \end{equation}
\end{dfn}
Ye~\etal find a solution~\cref{eq:flagpca} using Newton's method on the flag manifold ~\cite{ye2022optimization} and Nguyen offers a method for solving such a problem using RTR on flag manifolds~\cite{nguyen2022closed}. These algorithms produce the same basis vectors for flags regardless of flag type. These basis vectors are different than those found using standard PCA. But, for $[\![ \U ]\!] \in \flag(n_1,n_2,\dots,n_k,n)$ that solves~\cref{eq:flagpca} using either Newton's method or RTR, the column space of $\U_{:,:n_k}$ is the same as the span of the first $n_k$ principal components. This is because the objective function in~\cref{eq:flagpca} is invariant to ordering the columns of $\U$.

Variants on flagified PCA that maximize $\tr(\U^T\X\X^T \U)^q$ over $\flag(n+1)$ are coined ``nonlinear eigenflags'' and are difficult to solve for $q=2$~\cite{ye2022optimization}. Yet, methods from Mankovich~\etal can be adapted to solve such problems, especially for $q = 1/2$. Another variant of fPCA is weighted fPCA where we assume a weight for each subspace dimension in the flag $i$ and each data point $j$ as $w_{ij} \in \R$. We propose this formulation in the manuscript.

\section{DPCP-IRLS and the Grassmannian}\label{sec:grassdpcp}
This concept was first unerarthed in~\cite{mankovich2023subspace}. Expanding the matrix norm we have
\begin{align}
    \| \X^T \B \|_{1,2} &= \sum_{j=1}^p \| \B^T \x_j \|_2,\\
    &= \sum_{j=1}^p \sqrt{\sum_{i=1}^k | \vb_i^T \x_j|^2 }, \\
    &= \sum_{j=1}^p \sqrt{ \x_j^T \B \B^T \x_j }, \\
    &= \sum_{j=1}^p \sqrt{ \tr\left( \B^T \x_j  \x_j^T \B \right) }.
\end{align}
This can be phrased using principal angles as
\begin{equation}
    \argmin_{\B^T \B = \I} \sum_{j=1}^p \cos \theta([\x_j], [\B]).
\end{equation}
Suppose $\{[\X_j]\}_{j=1}^p \subset \Gr(k,n)$. Namely, $\X_j \in \R^{n \times k}$ where $\X_j^T \X_j = \I$ for each $j$. A natural generalization of DPCP-IRLS is the optimization problem on the Grassmannian,
\begin{equation}\label{eq:grassdpcp}
    \argmin_{[\B]\in \Gr(k,n)} \sum_{j=1}^p \| \cos \theta([\X_j], [\B])\|_2.
\end{equation}
This can also be solved by an IRLS scheme.

The ``flagified'' version of~\cref{eq:grassdpcp} is
\begin{equation}
    \argmin_{[\![ \B ]\!] \in \flag(n+1)} \sum_{j=1}^p \| \cos \theta([\X_j], [\B])\|_2.
\end{equation}

\section{Novel Flagified Robust and Dual PCA and TPCA Variants}

We present the intuition behind the geometry of Robust and Dual \PCA~versus \TPCA~in~\cref{fig: concept1}.  Then we provide a visual comparison between Euclidean and manifold variants of \RPCA~and \DPCP~in~\cref{fig: concept0}.

\begin{figure}[t]
    \centering
    \includegraphics[width = \linewidth]{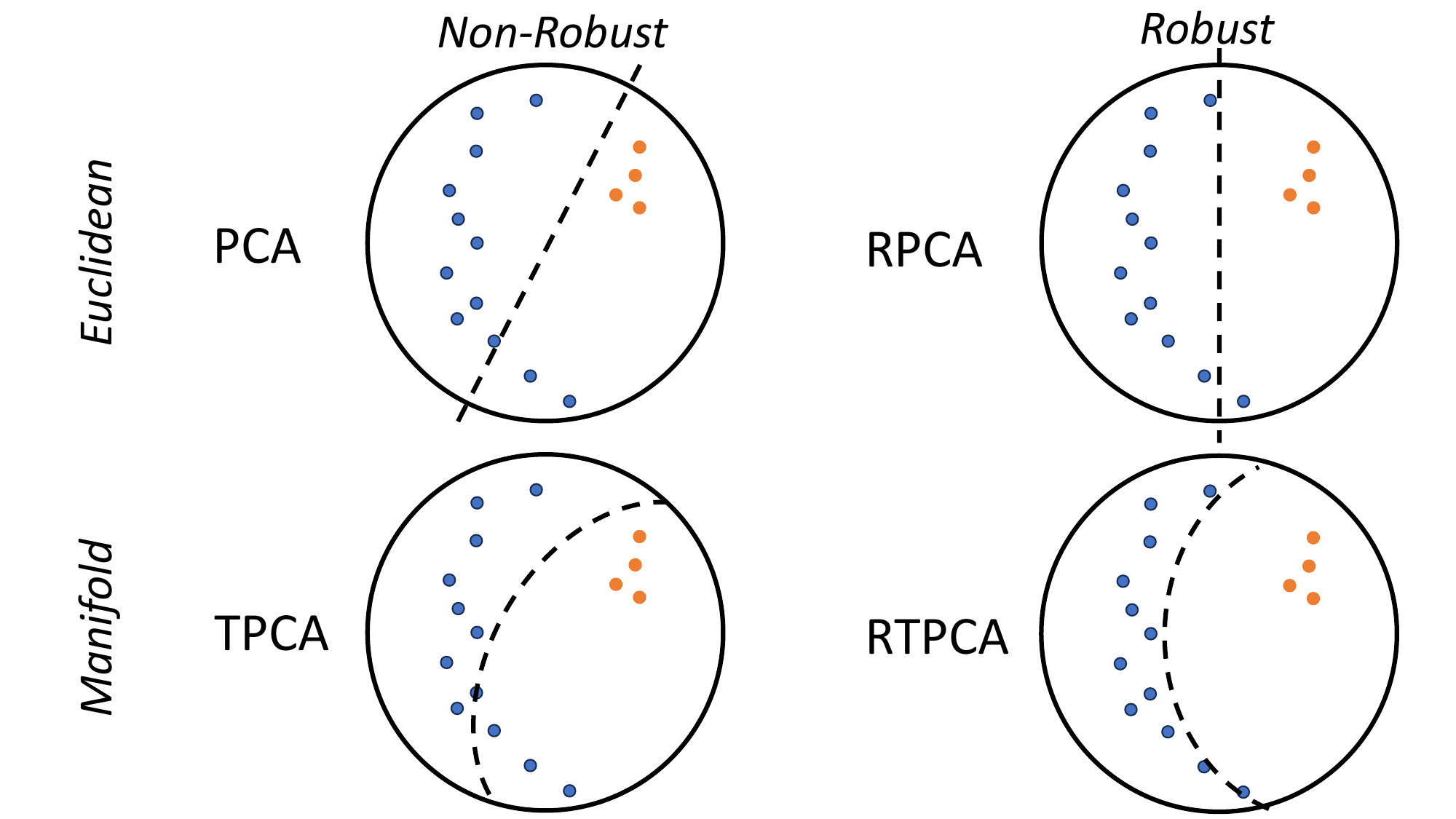}
    \caption{Inliers (blue) and outliers (orange) on the $2$-sphere. The first row are Euclidean algorithms and the second row are manifold (tangent space) algorithms. The dashed lines are the first principal subspace (first row) and geodesic (second row) spanned by the first principal direction. Note: first principal subspaces pass through the center of the sphere and first principal geodesics are great circles on the sphere.}
    \label{fig: concept1}
\end{figure}

\begin{figure}[t]
    \centering
    \includegraphics[width = \linewidth]{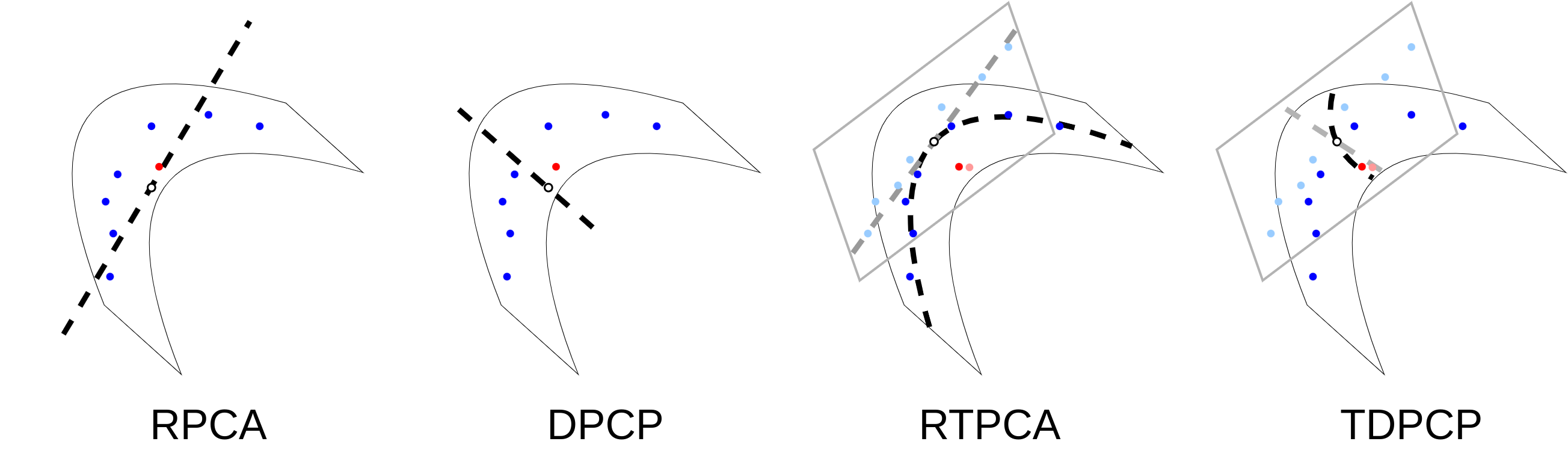}
    \caption{Given manifold valued data with inliers (blue) and outliers (red). The dashed black lines are the $1$st principal component for RPCA and DPCP, for R$\mathcal{T}$PCA and $\mathcal{T}$DPCP this is the $1$st principal geodesic. For RPCA and R$\mathcal{T}$PCA this line / geodesic should contain the inliers. Due to the reversal in the objective, for DPCP and $\mathcal{T}$DPCP this geodesic should contain the outliers.}
    \label{fig: concept0}
\end{figure}

\cref{tab:flags_rule} summarizes our novel flagified robust and dual PCA variants and emphasizes that flag types other than $(1,2\dots,k;n)$ and $(k;n)$ produce novel principal directions that are ``in between' $L_1$ and $L_2$ formulations.

\begin{table}[t]
    \centering
    \begin{tabular}{ccc}
        Flagified (Dual-)PCA & Robust PCA Variant\\
       \toprule
       \toprule
       \foneRPCA & \LoneRPCA\\
       \fRPCA$(\cdot)$ & -- \\
       \ftwoRPCA & \LtwoRPCA \\
       \midrule
       \foneWPCA & \LoneWPCA \\
       \fWPCA$(\cdot)$ & -- \\
       \ftwoWPCA & \LtwoWPCA \\
       \midrule
       \foneDPCP & \LoneDPCP\\
       \fDPCP$(\cdot)$ & --\\
       \ftwoDPCP & \LtwoDPCP\\
       \midrule
       \midrule
       \foneTRPCA & \LoneTRPCA\\
       \fTRPCA$(\cdot)$ & -- \\
       \ftwoTRPCA & \LtwoTRPCA \\
       \midrule
       \foneTWPCA & \LoneTWPCA \\
       \fTWPCA$(\cdot)$ & -- \\
       \ftwoTWPCA & \LtwoTWPCA \\
       \midrule
       \foneTDPCP & \LoneTDPCP\\
       \fTDPCP$(\cdot)$ & --\\
       \ftwoTDPCP & \LtwoTDPCP\\
    \end{tabular}
    \caption{Flag types for Euclidean optimization (first half) and manifold optimization (second half). Flag optimization in these algorithms provides a new objective functions which live in between $L_1$ and $L_2$ robust PCA formulations. Note: we remove the number of the ambient dimension in the flag signature for less redundant notation and we assume we are computing the first $k$ principal components.}
    \label{tab:flags_rule}
\end{table}

Finally~\cref{tab:alg_names} summarizes the naming schemes of all of the algorithms intriduced in this paper

\begin{table}[t]
    \centering
    \begin{tabular}{cc}
        Abbreviation & Name \\
        \toprule
        \toprule
        \PCA & Principal Component Analysis\\
        \RPCA & Robust \PCA\\
        \WPCA & Weiszfeld \PCA\\
        \DPCP & Dual Principal Component Pursuit\\
        \WDPCP & Weiszfeld DPCP\\
        \midrule
        \fPCA & Flagified PCA\\
        \fRPCA & Flagified \RPCA\\
        \fWPCA & Flagified \WPCA\\
        \fDPCP & Flagified \DPCP\\
        \fWDPCP & Flagified \WDPCP\\
        \midrule
        \TPCA & Tangent \PCA\\
        \TRPCA & Robust \TPCA\\
        \TWPCA & Weiszfeld \TPCA\\
        \TDPCP & Tangent \DPCP\\
        \TWDPCP & Tangent \WDPCP\\
        \midrule
        \fTPCA & Flagified \TPCA\\
        \fTRPCA & Flagified \TRPCA\\
        \fTWPCA & Flagified \TWPCA\\
        \fTDPCP & Flagified \TDPCP\\
        f\TWDPCP & Flagified \TWDPCP
    \end{tabular}
    \caption{The names of the major algorithms covered in this work.}
    \label{tab:alg_names}
\end{table}

\section{Rest of the Proposed Algorithms}
In the paper, we proposed three new algorithms. We now present these algorithms as well as the objective functions they minimize. First,~\cref{alg: weighted flagified pca} finds a solution to weighted flagified PCA
\begin{equation}
    [\![\U]\!]^\star=\argmax_{[\![\U]\!] \in \flag(n+1)}  \mathbb{E}_j \left[ \sum_{i=1}^k w_{ij}\| \pi_{\U_i} (\x_j)\|_2^2 \right].
\end{equation}
Second,~\cref{alg: weighted orth flagified pca} finds a solution to weighted flagified orthogonal PCA (\fOPCA)
\begin{equation}
    [\![\U]\!]^\star=\argmin_{[\![\U]\!] \in \flag(n+1)}  \mathbb{E}_j \left[ \sum_{i=1}^k w_{ij}\| \pi_{\U_i} (\x_j)\|_2^2 \right].
\end{equation}
Lastly,~\cref{alg: tPDs} approximates solutions to 
\begin{align}
&[\![\U]\!]^\star \approx \\
&\begin{cases}
        \argmax\limits_{[\![\U]\!] \in \flag(n+1)}  \mathbb{E}_j \left[ \sum_{i=1}^k w_{ij} d(\bm{\mu}, \pi_{\U_i} (\x_j))^2 \right], & (\text{\fTPCA})\\
        \argmin\limits_{[\![\U]\!] \in \flag(n+1)}  \mathbb{E}_j \left[ \sum_{i=1}^k w_{ij} d(\bm{\mu}, \pi_{\U_i} (\x_j))^2 \right], & (\text{\fTOPCA})\\
        \argmax\limits_{[\![\U]\!] \in \flag(n+1)}  \mathbb{E}_j \left[ \sum_{i=1}^k d(\bm{\mu}, \pi_{\U_i} (\x_j)) \right], & (\text{\fTRPCA})\\
        \argmin\limits_{[\![\U]\!] \in \flag(n+1)} \mathbb{E}_j \left[ \sum_{i=1}^k d( \x_j ,\pi_{\U_i} (\x_j)) \right], & (\text{\fTWPCA})\\
        \argmin\limits_{[\![\U]\!] \in \flag(n+1)}  \mathbb{E}_j \left[\sum_{i=1}^k d(\bm{\mu}, \pi_{\U_i} (\x_j)) \right], & (\text{\fTDPCP})
        \end{cases}\nonumber
\end{align}

\setlength{\floatsep}{0.1cm}
\begin{algorithm}[t]
\setstretch{1.13}
\caption{Weighted fPCA}\label{alg: weighted flagified pca}
 \textbf{Inputs}: {Dataset $\{\x_j \in \R^n\}_{j=1}^p $,\\ weights $\{w_{ij}\}_{i,j=1}^{i=k,j=p} \subset \R$,\\ flag type $(n+1)$}\\
 \textbf{Output}: Weighted flagified principal directions $[\![\U]\!]^\ast \in \flag(n+1)$ \\[0.25em]
 \For{$i=1,2,\dots,k$}{
    $(\W_i)_{jl} \gets \begin{cases}
        w_{ij}, & j= l\\
        0, &\text{ elsewhere}
         \end{cases}$
     }
 $ \U^\ast \gets$ \text{Solve }~\cref{eq:wflag_pca} with $\{\W_i\}_{i=1}^k$ via Stiefel-CGD.\\
 $[\![\U]\!]^\ast \gets [\![\U^\ast]\!]$
\end{algorithm}

\begin{algorithm}[t]
\setstretch{1.13}
\caption{Weighted flag \OPCA~(\fOPCA)}\label{alg: weighted orth flagified pca}
 \textbf{Inputs}: {Dataset $\{\x_j \in \R^n\}_{j=1}^p $,\\ weights $\{w_{ij}\}_{i,j=1}^{i=k,j=p} \subset \R$,\\ flag type $(n+1)$}\\
 \textbf{Output}: Weighted flagified principal directions $[\![\U]\!]^\ast \in \flag(n+1)$ \\[0.25em]
 \For{$i=1,2,\dots,k$}{
    $(\W_i)_{jl} \gets \begin{cases}
        w_{ij}, & j= l\\
        0, &\text{ elsewhere}
         \end{cases}$
     }
 $ \U^\ast \gets$ \text{Minimize the objective in }~\cref{eq:wflag_pca} with $\{\W_i\}_{i=1}^k$ via Stiefel-CGD.\\
 $[\![\U]\!]^\ast \gets [\![\U^\ast]\!]$
\end{algorithm}

\begin{algorithm}[t]
\setstretch{1.13}
\caption{\fTPCA/\fTRPCA/\fTWPCA/\fTDPCP}
\label{alg: tPDs}
     \textbf{Input}: {Dataset: $\{\x_j\}_{j=1}^p \subset \Man$, flag type $(n+1)$, \fPCA~Variant: $\Phi:\mathcal{W} \rightarrow \flag(n+1)$}\\ 
     \textbf{Output}: Flagified principal tangent directions $[ \! [ \U ] \! ]^\ast$ \\[0.25em] 
     \If{\text{robust}}{
         $\bm{\mu} \leftarrow \text{KarcherMedian}\left(\{\x_j\}_{j=1}^p\right)$
         }
     \Else{
         $\bm{\mu} \leftarrow \text{KarcherMean}\left(\{\x_j\}_{j=1}^p\right)$
         }
    $\{\bv_j \}_j     \gets \{\Exp_{\bmu}(\x_j)\}_j$\\
    $\mathcal{W} \gets \{ \text{vec} (\bv_j) \}_j$\\
    $[ \! [ \U ] \! ]^\ast \gets \Phi \left( \mathcal{W}, n+1 \right)$
\end{algorithm}

\section{Extra Experiments}
\paragraph{Impact of flag-type on cluster detection} 
To assess the impact of flag-type, we generate a 
dataset $\{\x_j\}_{j=1}^{300} \subset \R^{10}$ with $3$ clusters ($C_1,C_2,C_3$) in which we curate the flag type corresponding to the data structure: $\flag(2,5,7; 10)$. To do this we sample $\{\x_j\}_{j=1}^{300} \subset \R^{10}$ with $3$ clusters. The $l$th entry of $\x$, $(\x)_l \in \R$, is sampled from
\begin{align}
    C(1) &: (\x)_l \sim \begin{cases}
                            \mathcal{U}[0,1), & l\leq 2\\
                            \mathcal{U}[0,0.1), & l \geq 3
                        \end{cases},\\
    C(2) &: (\x)_l \sim \begin{cases}
                            \mathcal{U}[0,1), & 3 \leq i \leq 5\\
                            \mathcal{U}[0,0.1), & i \leq 2 \text{ or } i \geq 6
                        \end{cases},\\    
    C(3) &: (\x)_l \sim \begin{cases}
                            \mathcal{U}[0,1), & i=6,7\\
                            \mathcal{U}[0,0.1), & i \leq 5 \text{ or } i \geq 8
                        \end{cases}.             
\end{align}
We then compute $2$ sets of $k=7$ principal directions by running \fWPCA~with flag type $(2,5,7;10)$ (\fWPCA$(2,5,7)$) and \ftwoWPCA~using~\cref{alg: flagified pca} with $200$ max. iters. Both of these methods result in a flag representative $\U = [\U_1, \U_2, \U_3] \in \R^{10\times 7}$ where $\U_1 \in \R^{10 \times 2}$, $\U_2 \in \R^{10 \times 3}$, and $\U_3 \in \R^{10 \times 2}$. We compute the reconstruction error for point $j$ against each $\U_i$ as $\sum_{j=1}^p \|\x_j - \U_i \U_i^T \x_j\|_2$. These errors are used for $3$ classification tasks, predicting $C_i$ using $\U_i$ for $i=1,2,3$. 
The corresponding AUC values are in~\cref{tab:fl_type}.
\fWPCA$(2,5,7)$ produces higher AUCs 
because it is optimized over a more optimal flag type, respecting the subspace structure of the data.
\begin{table}[t!]
    \centering
    \begin{tabular}{l@{\:}||@{\:}c@{\:}|@{\:}c@{\:}|@{\:}c@{\:}|@{\:}c@{\:}|@{\:}c@{\:}|@{\:}c}
         & \multicolumn{2}{c}{Cluster $1$} & \multicolumn{2}{c}{Cluster $2$} & \multicolumn{2}{c}{Cluster $3$}\\
         \toprule
        \fWPCA$(\cdot)$ & $(7)$ & $(2,5,7)$ & $(7)$ & $(2,5,7)$ & $(7)$ & $(2,5,7)$ \\
        \hline
        AUC $\uparrow$ & $0.72$ & $\mathbf{0.73}$ & $0.48$ & $\mathbf{1.00}$ & $0.43$ & $\mathbf{0.49}$
    \end{tabular}
    \caption{AUC for cluster classification using  \fWPCA. We see higher AUCs when we match the flag type for \fWPCA~with the cluster dimensions (e.g., $(2,5,7)$).\vspace{-1mm}}
    \label{tab:fl_type}
\end{table}

\paragraph{Data generation for ``Convergence on $4$-sphere''}
We first sample a random center $\x \in \mathbb{S}^4$, and then sample $100$ inlier tangent vectors from $\mathcal{U}[0,.01)$. Another $20$ outlier tangent vectors $\bv$, have entries $v_1, v_2 \sim \mathcal{U}[0,.01)$ and $v_3, v_4, v_5 \sim \mathcal{U}[0,.1)$. We wrap these vectors to have our dataset, $\{\Exp_{\x}(\bv)\}$. 

\paragraph{Impact of flag type on principal directions}
We run flagified robust \PCA~and \TPCA~variants using~\cref{alg: flagified pca} (with $200$ max. iters.) with different flag types on data on $Gr(2,4)$ with $100$ inliers and $20$ outliers sampled as described in the ``Outlier detection on $Gr(2,4)$'' section of the manuscript. We call $(1,2,3,4;5)$ the ``base'' flag type. We use $T$ to measure the different between principal directions from the base flag type $\{\bu_1 \dots, \bu_4\}$ and other principal directions $\{\bv_1 \dots, \bv_4\}$ as
\begin{equation}
    T = \frac{1}{4} \sum_{i=1}^4 \theta ( \bu_i, \bv_i)^2.
\end{equation}
We plot $T$ values for different flagified robust \PCA~and \TPCA~variants in~\cref{fig: flag_type}. We separate flag types into classes based on the number of nested subspaces. Flag types with the same number of nested subspaces are considered ``closer'' flags. We find that closer flag types have smaller $T$ values. This experiment verifies that running flagified robust PCA variants with different flag types recover different principal directions and these differences are directly proportional to the ``distance'' between flag types. This also emphasizes that flag types other than $(1,\dots,k;n)$ and $(k;n)$ indeed recover novel principal directions. The direct utility of these gap-filling methods to real-world datasets is future work.
%

\begin{figure}[t]
    \centering
    \includegraphics[width =\linewidth]{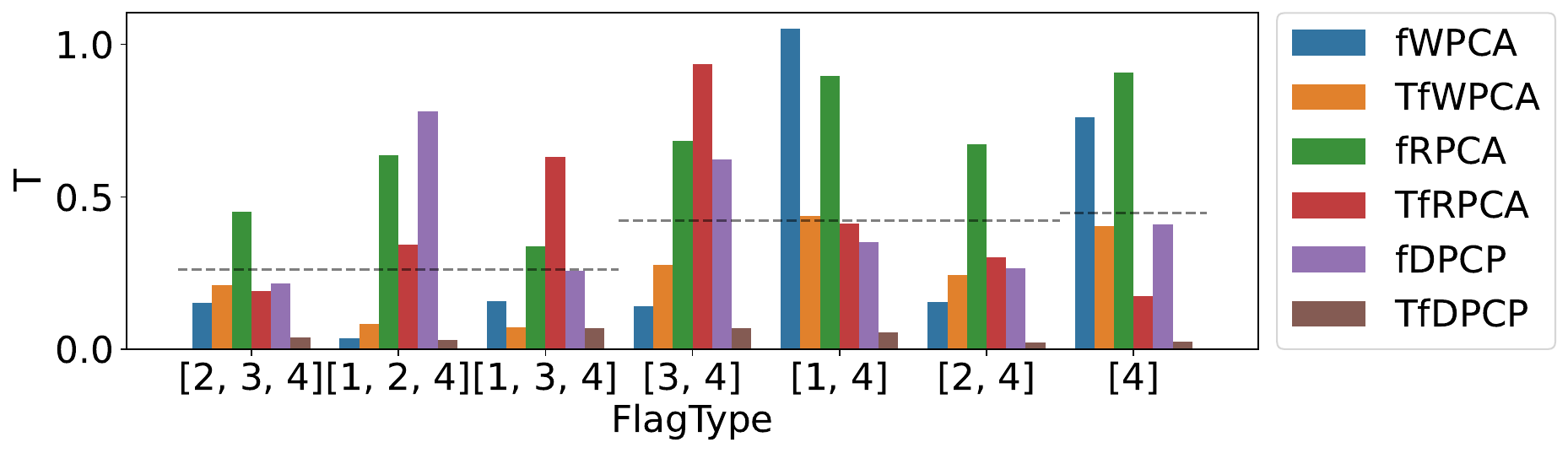}\vspace{-2mm}
    \caption{Smaller $T$ corresponds to principal directions which are more similar to those computed with flag type $(1,2,3,4;5)$. The mean $T$ values for each class of flag type are the horizontal dashed lines. Notice that, these mean values increase as we increase the distance between flag types. We truncate flag types by removing the ambient dimension ($5$).}
    \label{fig: flag_type}
\end{figure}

\paragraph{Outlier detection on $Gr(2,4)$}
We present the result of using \PCA, \foneWPCA, \ftwoWPCA, \foneRPCA, \ftwoRPCA, \foneDPCP, and \ftwoDPCP~on $Gr(2,4)$ data for outlier detection in~\cref{fig:out_auc_gr24_euc}. This is the same data as the data used for~\cref{fig:out_auc_gr24_man}; but, in this case, we run our algorithms on the vectorized matrix representatives for points on $Gr(2,4)$ and do outlier detection using Euclidean distance and variances.
\begin{figure}[t]
    \centering
    \includegraphics[width =\linewidth]{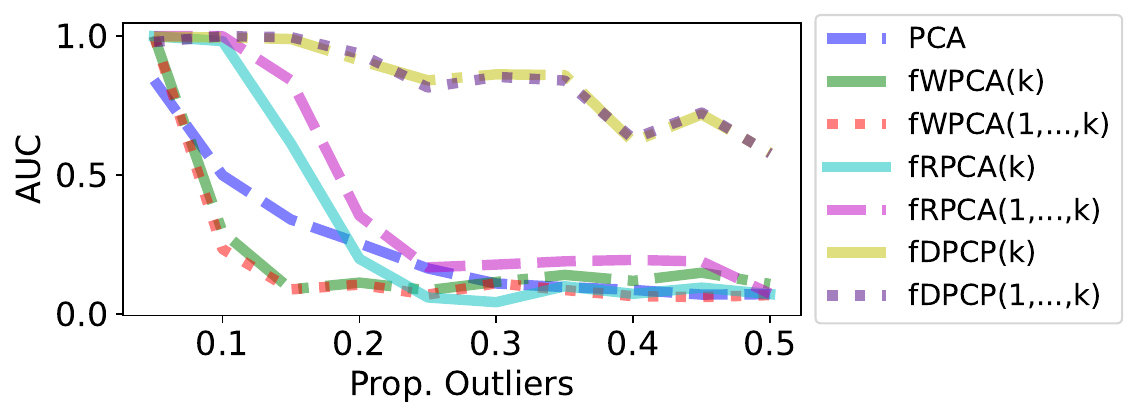}\vspace{-2mm}
gr    \caption{AUC of different algorithms for outlier detection using the first $k=2$ principal directions of outlier-contaminated data on $Gr(2,4)$. All algorithms other than \PCA~are optimized with~\cref{alg: flagified pca} with $100$ max. iters.\vspace{-2mm}}
    \label{fig:out_auc_gr24_euc}
\end{figure}

\paragraph{Hand reconstructions}
We use the $2$D Hands dataset and add ``hairball'' outliers by sampling from a normal distribution with mean $0$ and standard deviation $10$ ($\mathcal{N}(0,10)$), then we divide by the Frobenius norm and mean center to obtain a point on $\Sigma_2^{56}$. A figure with an example of an outlier ellipse and a hairball outlier is in~\cref{fig:outlier_qual}.

\begin{figure}[t]
    \centering
    \includegraphics[width = .7\linewidth]{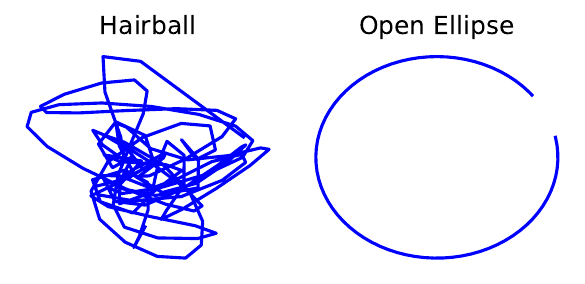}\vspace{-2mm}
    \caption{Examples of outliers used for contamination of the hands dataset. Hairballs are used in hand reconstruction and open ellipses are used in outlier detection.\vspace{-2mm}}
    \label{fig:outlier_qual}
\end{figure}

We run \foneTWPCA, \LoneTWPCA~using Alg. 1 from~\cite{wang2017ell} run on the tangent space, \foneTRPCA, and \TPCA~to find different versions of the first $k=1$ principal direction on a dataset with all $40$ hands and $5$ outliers. We compute reconstruction error for each method using the framework described in the $Gr(2,4)$ experiments. Our cumulative reconstruction errors for the $40$ inlier hands and a visualization of a hand reconstruction is in~\cref{fig:hand_res}. \LoneTWPCA~and \foneTWPCA~produce the lowest reconstruction errors on the hands and have the most sensible reconstructions. Additionally,~\cref{alg: tPDs} preforms just as well as Alg. 1 from~\cite{wang2017ell} run on the tangent space.

\begin{figure}[!t]
    \centering
    \includegraphics[width = .8\linewidth]{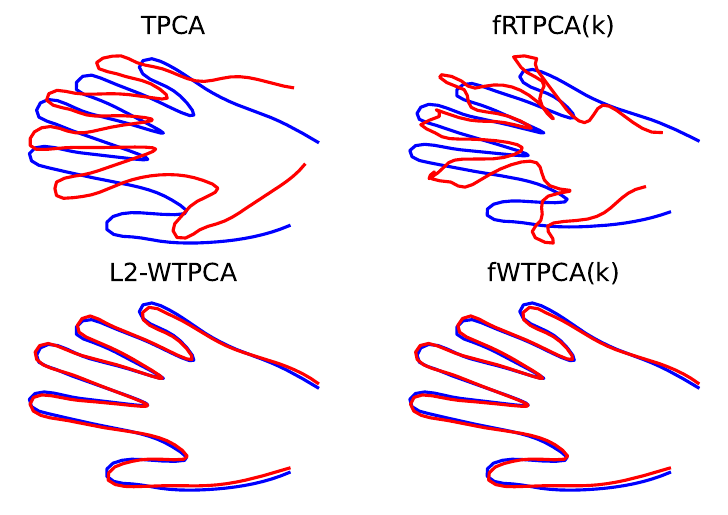}\vspace{-2mm}
    \caption{Reconstruction of hand $6$ using the first principal direction computed on a dataset with $40$ hands and $5$ outliers. The cumulative reconstruction errors for the $40$ inlier hands from L to R, Top to Bottom, are: $8.19$, $6.20$, $5.35$, and $5.35$.\vspace{-2mm}}
    \label{fig:hand_res}
\end{figure}

We move on to computing cumulative inlier reconstruction errors as we gradually add outliers and report results in~\cref{fig: hand_rec_ablation1,fig: hand_rec_ablation1}. \fTWPCA~have the most stable reconstruction errors followed by \fTRPCA, then \TPCA.

\begin{figure}[h!]
    \centering
    \includegraphics[width = \linewidth]{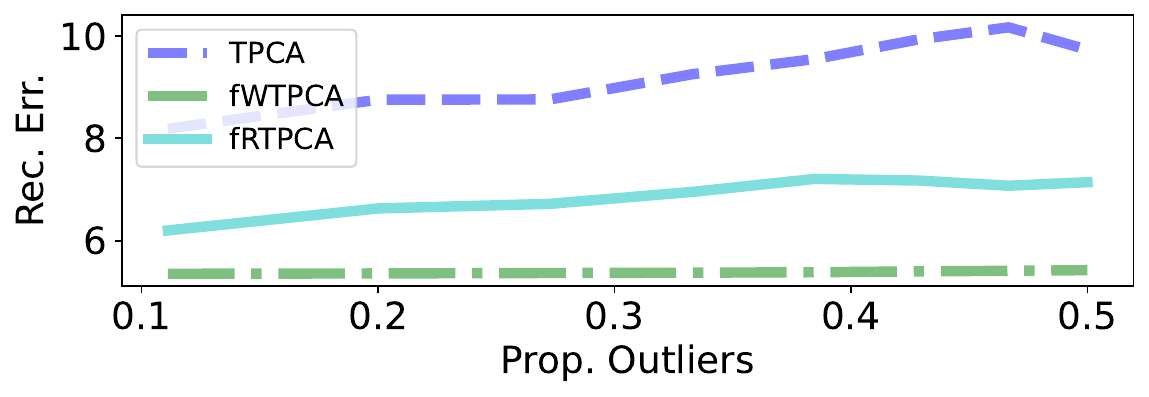}\vspace{-2mm}
    \caption{The cumulative reconstruction error of the $40$ inlier hands using the first $k=1$ principal direction where we gradually add hairball outliers.\vspace{-2mm}}
    \label{fig: hand_rec_ablation1}
\end{figure}

\begin{figure}[ht]
    \centering
    \includegraphics[width = \linewidth]{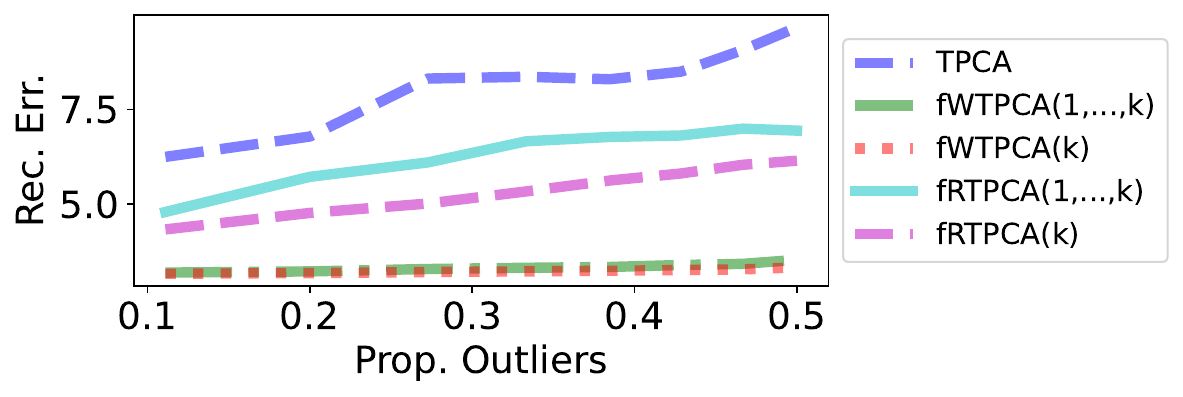}\vspace{-2mm}
    \caption{The cumulative reconstruction error of the $40$ inlier hands using the first $k=2$ principal directions where we gradually add hairball outliers.\vspace{-2mm}}
    \label{fig: hand_rec_ablation}
\end{figure}

\flushcolsend

%% file: main_arxiv.bbl
\begin{thebibliography}{80}
\providecommand{\natexlab}[1]{#1}
\providecommand{\url}[1]{\texttt{#1}}
\expandafter\ifx\csname urlstyle\endcsname\relax
  \providecommand{\doi}[1]{doi: #1}\else
  \providecommand{\doi}{doi: \begingroup \urlstyle{rm}\Url}\fi

\bibitem[Abboud et~al.(2020)Abboud, Benzinou, and Nasreddine]{abboud2020robust}
Michel Abboud, Abdesslam Benzinou, and Kamal Nasreddine.
\newblock A robust tangent {PCA} via shape restoration for shape variability
  analysis.
\newblock \emph{Pattern Analysis and Applications}, 23:\penalty0 653--671,
  2020.

\bibitem[Aftab and Hartley(2015)]{aftab2015convergence}
Khurrum Aftab and Richard Hartley.
\newblock Convergence of iteratively re-weighted least squares to robust
  m-estimators.
\newblock In \emph{IEEE Winter Conference on Applications of Computer Vision},
  pages 480--487. IEEE, 2015.

\bibitem[Aftab et~al.(2014)Aftab, Hartley, and Trumpf]{aftab2014generalized}
Khurrum Aftab, Richard Hartley, and Jochen Trumpf.
\newblock Generalized {W}eiszfeld algorithms for lq optimization.
\newblock \emph{IEEE Transactions on Pattern Analysis and Machine
  Intelligence}, 37\penalty0 (4):\penalty0 728--745, 2014.

\bibitem[Alekseevsky(1997)]{alekseevsky1997flag}
DV Alekseevsky.
\newblock Flag manifolds.
\newblock \emph{Sbornik Radova}, 11, 1997.

\bibitem[Arenas-Garc\'ia et~al.(2013)Arenas-Garc\'ia, Petersen, Camps-Valls,
  and Hansen]{arenas2013kernel}
Jer\'onimo Arenas-Garc\'ia, Kaare~Brandt Petersen, Gustavo Camps-Valls, and
  Lars~Kai Hansen.
\newblock Kernel multivariate analysis framework for supervised subspace
  learning: {A} tutorial on linear and kernel multivariate methods.
\newblock \emph{IEEE Signal Processing Magazine}, 30\penalty0 (4):\penalty0
  16--29, 2013.

\bibitem[Awate et~al.(2014)Awate, Yu, and Whitaker]{awate2014kernel}
Suyash~P Awate, Yen-Yun Yu, and Ross~T Whitaker.
\newblock Kernel principal geodesic analysis.
\newblock In \emph{ECML}, pages 82--98. Springer, 2014.

\bibitem[Beck and Sabach(2015)]{beck2015weiszfeld}
Amir Beck and Shoham Sabach.
\newblock Weiszfeld’s method: {O}ld and new results.
\newblock \emph{Journal of Optimization Theory and Applications}, 164:\penalty0
  1--40, 2015.

\bibitem[Boumal(2023)]{boumal2023intromanifolds}
Nicolas Boumal.
\newblock \emph{An introduction to optimization on smooth manifolds}.
\newblock Cambridge University Press, 2023.

\bibitem[Buet and Pennec(2023)]{buet2023flagfolds}
Blanche Buet and Xavier Pennec.
\newblock Flagfolds.
\newblock \emph{arXiv preprint arXiv:2305.10583}, 2023.

\bibitem[Busam et~al.(2017)Busam, Birdal, and Navab]{busam2017camera}
Benjamin Busam, Tolga Birdal, and Nassir Navab.
\newblock Camera pose filtering with local regression geodesics on the
  {R}iemannian manifold of dual quaternions.
\newblock In \emph{CVPR}, pages 2436--2445, 2017.

\bibitem[Cand{\`e}s et~al.(2011)Cand{\`e}s, Li, Ma, and
  Wright]{candes2011robust}
Emmanuel~J Cand{\`e}s, Xiaodong Li, Yi Ma, and John Wright.
\newblock Robust principal component analysis?
\newblock \emph{Journal of the ACM}, 58\penalty0 (3):\penalty0 1--37, 2011.

\bibitem[Ciuclea et~al.(2023)Ciuclea, Tumpach, and Vizman]{ciuclea2023shape}
Ioana Ciuclea, Alice~Barbora Tumpach, and Cornelia Vizman.
\newblock Shape spaces of nonlinear flags.
\newblock In \emph{International Conference on Geometric Science of
  Information}, pages 41--50. Springer, 2023.

\bibitem[Damon and Marron(2014)]{damon2014backwards}
James Damon and JS Marron.
\newblock Backwards principal component analysis and principal nested
  relations.
\newblock \emph{Journal of Mathematical Imaging and Vision}, 50\penalty0 (1-2),
  2014.

\bibitem[Ding et~al.(2006)Ding, Zhou, He, and Zha]{ding2006r}
Chris Ding, Ding Zhou, Xiaofeng He, and Hongyuan Zha.
\newblock {R}1-{PCA}: {R}otational invariant {L}1-norm principal component
  analysis for robust subspace factorization.
\newblock In \emph{International Conference on Machine learning}, pages
  281--288, 2006.

\bibitem[Donagi and Sharpe(2008)]{donagi2008glsms}
Ron Donagi and Eric Sharpe.
\newblock Glsms for partial flag manifolds.
\newblock \emph{Journal of Geometry and Physics}, 58\penalty0 (12), 2008.

\bibitem[Draper et~al.(2014)Draper, Kirby, Marks, Marrinan, and
  Peterson]{draper2014flag}
Bruce Draper, Michael Kirby, Justin Marks, Tim Marrinan, and Chris Peterson.
\newblock A flag representation for finite collections of subspaces of mixed
  dimensions.
\newblock \emph{Linear Algebra and its Applications}, 451:\penalty0 15--32,
  2014.

\bibitem[Dryden et~al.(2019)Dryden, Kim, Laughton, and Le]{dryden2019principal}
Ian~L Dryden, Kwang-Rae Kim, Charles~A Laughton, and Huiling Le.
\newblock Principal nested shape space analysis of molecular dynamics data.
\newblock \emph{The Annals of Applied Statistics}, 13\penalty0 (4):\penalty0
  2213--2234, 2019.

\bibitem[Edelman et~al.(1998)Edelman, Arias, and Smith]{edelman1998geometry}
Alan Edelman, Tom{\'a}s~A Arias, and Steven~T Smith.
\newblock The geometry of algorithms with orthogonality constraints.
\newblock \emph{SIAM Journal on Matrix Analysis and Applications}, 20\penalty0
  (2):\penalty0 303--353, 1998.

\bibitem[Fan et~al.(2022)Fan, Yang, and Vemuri]{fan2022nested}
Xiran Fan, Chun-Hao Yang, and Baba~C Vemuri.
\newblock Nested hyperbolic spaces for dimensionality reduction and hyperbolic
  nn design.
\newblock In \emph{CVPR}, 2022.

\bibitem[Fletcher et~al.(2003)Fletcher, Lu, and Joshi]{fletcher2003statistics}
P~Thomas Fletcher, Conglin Lu, and Sarang Joshi.
\newblock Statistics of shape via principal geodesic analysis on {L}ie groups.
\newblock In \emph{CVPR}. IEEE, 2003.

\bibitem[Fletcher et~al.(2004)Fletcher, Lu, Pizer, and
  Joshi]{fletcher2004principal}
P~Thomas Fletcher, Conglin Lu, Stephen~M Pizer, and Sarang Joshi.
\newblock Principal geodesic analysis for the study of nonlinear statistics of
  shape.
\newblock \emph{IEEE Transactions on Medical Imaging}, 23\penalty0
  (8):\penalty0 995--1005, 2004.

\bibitem[Hager and Zhang(2006)]{hager2006survey}
William~W Hager and Hongchao Zhang.
\newblock A survey of nonlinear conjugate gradient methods.
\newblock \emph{Pacific Journal of Optimization}, 2\penalty0 (1):\penalty0
  35--58, 2006.

\bibitem[Haller and Vizman(2020)]{haller2020nonlinear}
Stefan Haller and Cornelia Vizman.
\newblock Nonlinear flag manifolds as coadjoint orbits.
\newblock \emph{Annals of global analysis and geometry}, 58\penalty0
  (4):\penalty0 385--413, 2020.

\bibitem[Haller and Vizman(2023)]{haller2023weighted}
Stefan Haller and Cornelia Vizman.
\newblock Weighted nonlinear flag manifolds as coadjoint orbits.
\newblock \emph{arXiv preprint arXiv:2301.00428}, 2023.

\bibitem[Harandi et~al.(2014)Harandi, Salzmann, and
  Hartley]{harandi2014manifold}
Mehrtash~T Harandi, Mathieu Salzmann, and Richard Hartley.
\newblock From manifold to manifold: {G}eometry-aware dimensionality reduction
  for {SPD} matrices.
\newblock In \emph{ECCV}, pages 17--32. Springer, 2014.

\bibitem[Hastie and Stuetzle(1989)]{hastie1989principal}
Trevor Hastie and Werner Stuetzle.
\newblock Principal curves.
\newblock \emph{Journal of the American Statistical Association}, 84\penalty0
  (406), 1989.

\bibitem[Hotelling(1933)]{hotelling1933analysis}
Harold Hotelling.
\newblock Analysis of a complex of statistical variables into principal
  components.
\newblock \emph{Journal of Educational Psychology}, 24\penalty0 (6):\penalty0
  417, 1933.

\bibitem[Huang and Pan(2020)]{huang2020communication}
Long-Kai Huang and Sinno Pan.
\newblock Communication-efficient distributed {PCA} by {R}iemannian
  optimization.
\newblock In \emph{International Conference on Machine Learning}, pages
  4465--4474. PMLR, 2020.

\bibitem[Huang and Wei(2022)]{huang2022extension}
Wen Huang and Ke Wei.
\newblock An extension of fast iterative shrinkage-thresholding algorithm to
  {R}iemannian optimization for sparse principal component analysis.
\newblock \emph{Numerical Linear Algebra with Applications}, 29\penalty0
  (1):\penalty0 e2409, 2022.

\bibitem[Huckemann and Ziezold(2006)]{huckemann2006principal}
Stephan Huckemann and Herbert Ziezold.
\newblock Principal component analysis for {R}iemannian manifolds, with an
  application to triangular shape spaces.
\newblock \emph{Advances in Applied Probability}, 38\penalty0 (2):\penalty0
  299--319, 2006.

\bibitem[Huckemann et~al.(2010)Huckemann, Hotz, and
  Munk]{huckemann2010intrinsic}
Stephan Huckemann, Thomas Hotz, and Axel Munk.
\newblock Intrinsic shape analysis: {G}eodesic {PCA} for {R}iemannian manifolds
  modulo isometric {L}ie group actions.
\newblock \emph{Statistica Sinica}, pages 1--58, 2010.

\bibitem[Jung et~al.(2012)Jung, Dryden, and Marron]{jung2012analysis}
Sungkyu Jung, Ian~L Dryden, and James~Stephen Marron.
\newblock Analysis of principal nested spheres.
\newblock \emph{Biometrika}, 99\penalty0 (3):\penalty0 551--568, 2012.

\bibitem[Kendall(1984)]{kendall1984shape}
David~G Kendall.
\newblock Shape manifolds, procrustean metrics, and complex projective spaces.
\newblock \emph{Bulletin of the London Mathematical Society}, 16\penalty0
  (2):\penalty0 81--121, 1984.

\bibitem[Kirby(2001)]{kirby2001geometric}
Michael Kirby.
\newblock \emph{Geometric data analysis: {A}n empirical approach to
  dimensionality reduction and the study of patterns}.
\newblock Wiley New York, 2001.

\bibitem[Kwak(2008)]{kwak2008principal}
Nojun Kwak.
\newblock Principal component analysis based on l1-norm maximization.
\newblock \emph{IEEE Transactions on Pattern Analysis and Machine
  Intelligence}, 30\penalty0 (9):\penalty0 1672--1680, 2008.

\bibitem[Kwak(2013)]{kwak2013principal}
Nojun Kwak.
\newblock Principal component analysis by ${L}_p$-norm maximization.
\newblock \emph{IEEE Transactions on Cybernetics}, 44\penalty0 (5):\penalty0
  594--609, 2013.

\bibitem[Laparra et~al.(2012)Laparra, Jim{\'e}nez, Camps-Valls, and
  Malo]{laparra2012nonlinearities}
Valero Laparra, S Jim{\'e}nez, Gustavo Camps-Valls, and Jes\'us Malo.
\newblock Nonlinearities and adaptation of color vision from sequential
  principal curves analysis.
\newblock \emph{Neural Computation}, 24\penalty0 (10):\penalty0 2751--2788,
  2012.

\bibitem[Lee(2006)]{lee2006riemannian}
John~M Lee.
\newblock \emph{Riemannian manifolds: {A}n introduction to curvature}.
\newblock Springer Science \& Business Media, 2006.

\bibitem[Lerman and Maunu(2018)]{lerman2018fast}
Gilad Lerman and Tyler Maunu.
\newblock Fast, robust and non-convex subspace recovery.
\newblock \emph{Information and Inference: A Journal of the IMA}, 7\penalty0
  (2):\penalty0 277--336, 2018.

\bibitem[Li et~al.(2003)Li, Andreetto, and Ranzato]{caltech101}
Fei-Fei Li, Marco Andreetto, and Marc~'Aurelio Ranzato.
\newblock Caltech101 image dataset.
\newblock 2003.

\bibitem[Lin et~al.(2020)Lin, Lazar, Sarpabayeva, and Dunson]{lin2020robust}
Lizhen Lin, Drew Lazar, Bayan Sarpabayeva, and David~B Dunson.
\newblock Robust optimization and inference on manifolds.
\newblock \emph{arXiv preprint arXiv:2006.06843}, 2020.

\bibitem[Mankovich and Birdal(2023)]{Mankovich_2023_ICCV}
Nathan Mankovich and Tolga Birdal.
\newblock Chordal averaging on flag manifolds and its applications.
\newblock In \emph{ICCV}, pages 3881--3890, 2023.

\bibitem[Mankovich et~al.(2022)Mankovich, King, Peterson, and
  Kirby]{mankovich2022flag}
Nathan Mankovich, Emily~J King, Chris Peterson, and Michael Kirby.
\newblock The flag median and {F}lag{IRLS}.
\newblock In \emph{CVPR}, pages 10339--10347, 2022.

\bibitem[Mankovich(2023)]{mankovich2023subspace}
Nathan~J Mankovich.
\newblock \emph{Subspace and Network Averaging for Computer Vision and
  Bioinformatics}.
\newblock PhD thesis, Colorado State University, 2023.

\bibitem[Markopoulos et~al.(2014)Markopoulos, Karystinos, and
  Pados]{markopoulos2014optimal}
Panos~P Markopoulos, George~N Karystinos, and Dimitris~A Pados.
\newblock Optimal algorithms for ${L}_1$-subspace signal processing.
\newblock \emph{IEEE Transactions on Signal Processing}, 62\penalty0
  (19):\penalty0 5046--5058, 2014.

\bibitem[Markopoulos et~al.(2017)Markopoulos, Kundu, Chamadia, and
  Pados]{markopoulos2017efficient}
Panos~P Markopoulos, Sandipan Kundu, Shubham Chamadia, and Dimitris~A Pados.
\newblock Efficient {L1}-norm principal-component analysis via bit flipping.
\newblock \emph{IEEE Transactions on Signal Processing}, 65\penalty0
  (16):\penalty0 4252--4264, 2017.

\bibitem[Neumayer et~al.(2020)Neumayer, Nimmer, Setzer, and
  Steidl]{neumayer2020robust}
Sebastian Neumayer, Max Nimmer, Simon Setzer, and Gabriele Steidl.
\newblock On the robust {PCA} and {W}eiszfeld’s algorithm.
\newblock \emph{Applied Mathematics \& Optimization}, 82\penalty0 (3):\penalty0
  1017--1048, 2020.

\bibitem[Nguyen(2022)]{nguyen2022closed}
Du Nguyen.
\newblock Closed-form geodesics and optimization for {R}iemannian logarithms of
  {S}tiefel and flag manifolds.
\newblock \emph{Journal of Optimization Theory and Applications}, 194\penalty0
  (1), 2022.

\bibitem[Nguyen(2023)]{nguyen2023operator}
Du Nguyen.
\newblock Operator-valued formulas for {R}iemannian gradient and hessian and
  families of tractable metrics in {R}iemannian optimization.
\newblock \emph{Journal of Optimization Theory and Applications}, pages 1--30,
  2023.

\bibitem[Nishimori et~al.(2006{\natexlab{a}})Nishimori, Akaho, and
  Plumbley]{nishimori2006riemannian0}
Yasunori Nishimori, Shotaro Akaho, and Mark~D Plumbley.
\newblock {R}iemannian optimization method on generalized flag manifolds for
  complex and subspace {ICA}.
\newblock In \emph{AIP Conference}, 2006{\natexlab{a}}.

\bibitem[Nishimori et~al.(2006{\natexlab{b}})Nishimori, Akaho, and
  Plumbley]{nishimori2006riemannian1}
Yasunori Nishimori, Shotaro Akaho, and Mark~D Plumbley.
\newblock Riemannian optimization method on the flag manifold for independent
  subspace analysis.
\newblock In \emph{International conference on independent component analysis
  and signal separation}, pages 295--302. Springer, 2006{\natexlab{b}}.

\bibitem[Nishimori et~al.(2007)Nishimori, Akaho, Abdallah, and
  Plumbley]{nishimori2007flag}
Yasunori Nishimori, Shotaro Akaho, Samer Abdallah, and Mark~D Plumbley.
\newblock Flag manifolds for subspace {ICA} problems.
\newblock In \emph{ICASSP}, pages IV--1417. IEEE, 2007.

\bibitem[Nishimori et~al.(2008)Nishimori, Akaho, and
  Plumbley]{nishimori2008natural}
Yasunori Nishimori, Shotaro Akaho, and Mark~D Plumbley.
\newblock Natural conjugate gradient on complex flag manifolds for complex
  independent subspace analysis.
\newblock In \emph{International Conference on Artificial Neural Networks}.
  Springer, 2008.

\bibitem[Pearson(1901)]{pearson1901liii}
Karl Pearson.
\newblock {LIII. On lines and planes of closest fit to systems of points in
  space}.
\newblock \emph{The London, Edinburgh, and Dublin Philosophical Magazine and
  Journal of Science}, 1901.

\bibitem[Peng et~al.(2023)Peng, K{\"u}mmerle, and Vidal]{peng2023convergence}
Liangzu Peng, Christian K{\"u}mmerle, and Ren{\'e} Vidal.
\newblock On the {C}onvergence of {IRLS} and its {V}ariants in
  {O}utlier-{R}obust {E}stimation.
\newblock In \emph{CVPR}, pages 17808--17818, 2023.

\bibitem[Pennec(2018)]{pennec2018barycentric}
Xavier Pennec.
\newblock Barycentric subspace analysis on manifolds.
\newblock \emph{Annals of Statistics}, 46\penalty0 (6A), 2018.

\bibitem[Pennec(2020)]{pennec2020advances}
Xavier Pennec.
\newblock Advances in geometric statistics for manifold dimension reduction.
\newblock \emph{Handbook of Variational Methods for Nonlinear Geometric Data},
  pages 339--359, 2020.

\bibitem[Pitaval and Tirkkonen(2013)]{pitaval2013flag}
Renaud-Alexandre Pitaval and Olav Tirkkonen.
\newblock Flag orbit codes and their expansion to {S}tiefel codes.
\newblock In \emph{IEEE Information Theory Workshop}, pages 1--5. IEEE, 2013.

\bibitem[Polyak and Khlebnikov(2017)]{polyak2017robust}
Boris~T Polyak and Mikhail~V Khlebnikov.
\newblock Robust principal component analysis: {A}n {IRLS} approach.
\newblock \emph{IFAC-PapersOnLine}, 50\penalty0 (1):\penalty0 2762--2767, 2017.

\bibitem[Rabenoro and Pennec(2022)]{rabenoro2022geometric}
Dimbihery Rabenoro and Xavier Pennec.
\newblock A geometric framework for asymptotic inference of principal subspaces
  in {PCA}.
\newblock \emph{arXiv preprint arXiv:2209.02025}, 2022.

\bibitem[Said et~al.(2007)Said, Courty, Le~Bihan, and Sangwine]{said2007exact}
Salem Said, Nicolas Courty, Nicolas Le~Bihan, and Stephen~J Sangwine.
\newblock Exact principal geodesic analysis for data on so (3).
\newblock In \emph{European Signal Processing Conference}, pages 1701--1705.
  IEEE, 2007.

\bibitem[Sato(2022)]{sato2022riemannian}
Hiroyuki Sato.
\newblock Riemannian conjugate gradient methods: {G}eneral framework and
  specific algorithms with convergence analyses.
\newblock \emph{SIAM Journal on Optimization}, 32\penalty0 (4):\penalty0
  2690--2717, 2022.

\bibitem[Sch{\"o}lkopf et~al.(1997)Sch{\"o}lkopf, Smola, and
  M{\"u}ller]{scholkopf1997kernel}
Bernhard Sch{\"o}lkopf, Alexander Smola, and Klaus-Robert M{\"u}ller.
\newblock Kernel principal component analysis.
\newblock In \emph{International Conference on Artificial Neural Networks},
  pages 583--588. Springer, 1997.

\bibitem[Smith et~al.(2022)Smith, Laubach, Castillo, and Zavala]{smith2022data}
Alexander Smith, Benjamin Laubach, Ivan Castillo, and Victor~M Zavala.
\newblock Data analysis using {R}iemannian geometry and applications to
  chemical engineering.
\newblock \emph{Computers \& Chemical Engineering}, 168:\penalty0 108023, 2022.

\bibitem[Sommer et~al.(2010)Sommer, Lauze, Hauberg, and
  Nielsen]{sommer2010manifold}
Stefan Sommer, Fran{\c{c}}ois Lauze, S{\o}ren Hauberg, and Mads Nielsen.
\newblock Manifold valued statistics, exact principal geodesic analysis and the
  effect of linear approximations.
\newblock In \emph{ECCV}, pages 43--56. Springer, 2010.

\bibitem[Sommer et~al.(2014)Sommer, Lauze, and Nielsen]{sommer2014optimization}
Stefan Sommer, Fran{\c{c}}ois Lauze, and Mads Nielsen.
\newblock Optimization over geodesics for exact principal geodesic analysis.
\newblock \emph{Advances in Computational Mathematics}, 40, 2014.

\bibitem[Stegmann and Gomez(2002)]{stegmann2002}
M.~B. Stegmann and D.~D. Gomez.
\newblock A brief introduction to statistical shape analysis.
\newblock page~15, 2002.

\bibitem[Szwagier and Pennec(2023)]{szwagier2023rethinking}
Tom Szwagier and Xavier Pennec.
\newblock Rethinking the {R}iemannian logarithm on flag manifolds as an
  orthogonal alignment problem.
\newblock pages 375--383, 2023.

\bibitem[Tabaghi et~al.(2023)Tabaghi, Khanzadeh, Wang, and
  Mirarab]{tabaghi2023principal}
Puoya Tabaghi, Michael Khanzadeh, Yusu Wang, and Sivash Mirarab.
\newblock Principal component analysis in space forms.
\newblock \emph{arXiv preprint arXiv:2301.02750}, 2023.

\bibitem[Townsend et~al.(2016)Townsend, Koep, and
  Weichwald]{townsend2016pymanopt}
James Townsend, Niklas Koep, and Sebastian Weichwald.
\newblock Pymanopt: {A} {P}ython toolbox for optimization on manifolds using
  automatic differentiation.
\newblock \emph{arXiv preprint arXiv:1603.03236}, 2016.

\bibitem[Tsakiris and Vidal(2018)]{vidal2018dpcp}
Manolis Tsakiris and Ren\'e Vidal.
\newblock Dual principal component pursuit.
\newblock \emph{Journal of Machine Learning Research}, pages 1–--50, 2018.

\bibitem[Tsakiris and Vidal(2015)]{Tsakiris_2015_ICCV_Workshops}
Manolis~C. Tsakiris and Rene Vidal.
\newblock Dual principal component pursuit.
\newblock In \emph{ICCV}, 2015.

\bibitem[Vidal et~al.(2005)Vidal, Ma, and Sastry]{vidal2005generalized}
Ren\'e Vidal, Yi Ma, and Shankar Sastry.
\newblock Generalized principal component analysis ({GPCA}).
\newblock \emph{IEEE Transactions on Pattern Analysis and Machine
  Intelligence}, 27\penalty0 (12), 2005.

\bibitem[Wang et~al.(2017)Wang, Gao, Gao, and Nie]{wang2017ell}
Qianqian Wang, Quanxue Gao, Xinbo Gao, and Feiping Nie.
\newblock $\ell_{2, p}$-norm based {PCA} for image recognition.
\newblock \emph{IEEE Transactions on Image Processing}, 27\penalty0
  (3):\penalty0 1336--1346, 2017.

\bibitem[Wang et~al.(2023)Wang, Nie, Wang, Wang, and Li]{wang2023max}
Sisi Wang, Feiping Nie, Zheng Wang, Rong Wang, and Xuelong Li.
\newblock Max--min robust principal component analysis.
\newblock \emph{Neurocomputing}, 521:\penalty0 89--98, 2023.

\bibitem[Wiggerman(1998)]{wiggerman1998fundamental}
Mark Wiggerman.
\newblock The fundamental group of a real flag manifold.
\newblock \emph{Indagationes Mathematicae}, 9\penalty0 (1):\penalty0 141--153,
  1998.

\bibitem[Yale(2001)]{cropped_yaleb}
Yale.
\newblock The extended yale face database b (cropped).
\newblock 2001.

\bibitem[Yang and Newsam(2010)]{yang2010bag}
Yi Yang and Shawn Newsam.
\newblock Bag-of-visual-words and spatial extensions for land-use
  classification.
\newblock In \emph{Proceedings of the 18th SIGSPATIAL international conference
  on advances in geographic information systems}, pages 270--279, 2010.

\bibitem[Ye et~al.(2022)Ye, Wong, and Lim]{ye2022optimization}
Ke Ye, Ken Sze-Wai Wong, and Lek-Heng Lim.
\newblock Optimization on flag manifolds.
\newblock \emph{Mathematical Programming}, 194\penalty0 (1):\penalty0 621--660,
  2022.

\bibitem[Zhang and Yang(2018)]{zhang2018robust}
Teng Zhang and Yi Yang.
\newblock Robust {PCA} by manifold optimization.
\newblock \emph{The Journal of Machine Learning Research}, 19\penalty0
  (1):\penalty0 3101--3139, 2018.

\end{thebibliography}
